\documentclass[10pt,twocolumn,letterpaper]{article}

\usepackage[pagenumbers]{iccv} % To force page numbers, e.g. for an arXiv version

\usepackage{caption}
\usepackage{graphicx}
\usepackage{cuted}    % For the strip environment
\usepackage{multirow}
\usepackage{algorithm}
\usepackage{algorithmicx}
\usepackage{algpseudocode}
\usepackage{amsthm}
\usepackage{amsmath}
\newtheorem{theorem}{Theorem}
\newtheorem{lemma}[theorem]{Lemma}

\usepackage{array}

\definecolor{iccvblue}{rgb}{0.21,0.49,0.74}
\usepackage[pagebackref,breaklinks,colorlinks,allcolors=iccvblue]{hyperref}

\def\paperID{1506} % *** Enter the Paper ID here
\def\confName{ICCV}
\def\confYear{2025}

\title{MotionPCM: Real-Time Motion Synthesis with Phased Consistency Model}

\author{Lei Jiang\\
University College London\\
 London, UK\\
{\tt\small lei.j@ucl.ac.uk}
\and
Ye Wei\\
University of Oxford\\
Oxford, UK\\
{\tt\small ye.wei@ndcls.ox.ac.uk}
\and
Hao Ni\\
University College London\\
 London, UK\\
{\tt\small h.ni@ucl.ac.uk}
}

\begin{document}

\maketitle

\begin{strip}
    \centering
    \includegraphics[width=0.82\textwidth]{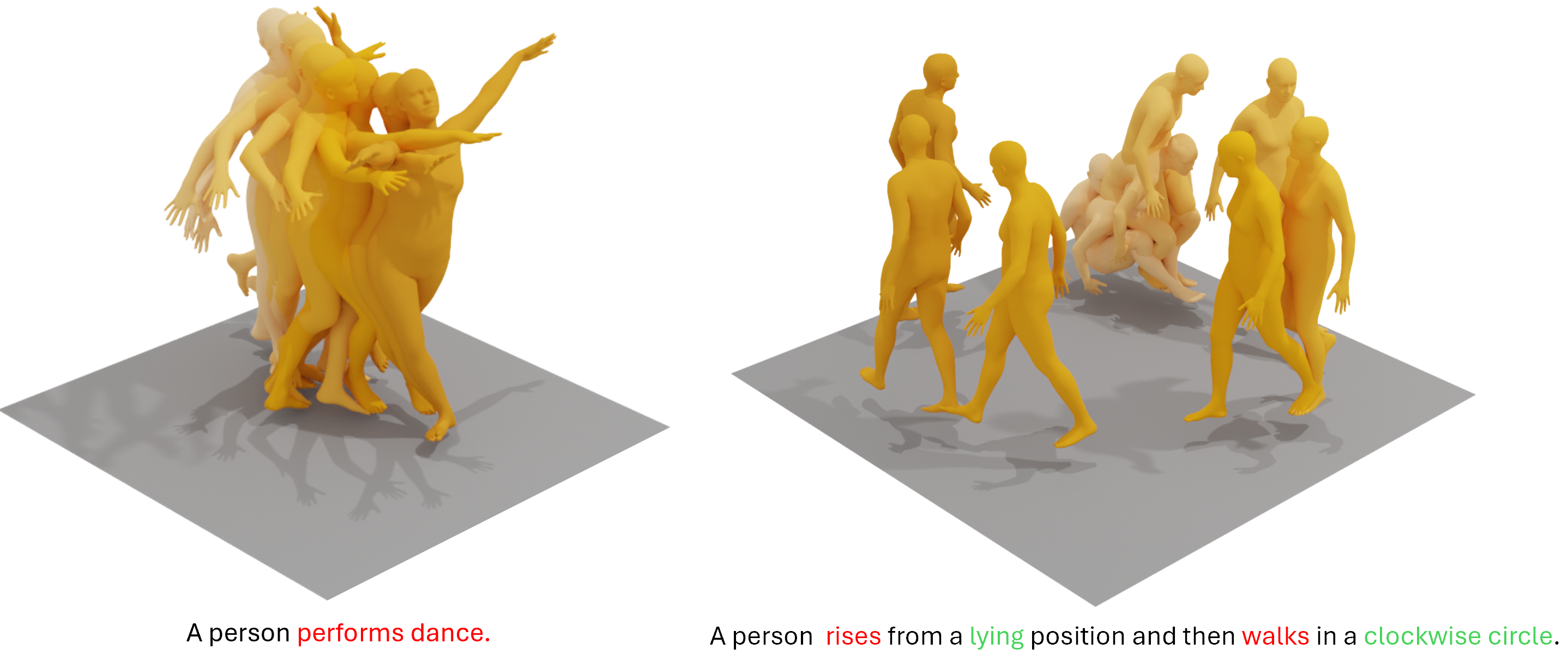}
    \captionof{figure}{We propose a new text-conditioned motion synthesis model: MotionPCM, capable of real-time motion generation with improved performance. Lighter colours represent earlier time points.}
    % \label{fig:example-figure}
\end{strip}

\begin{abstract}
Diffusion models have become a popular choice for human motion synthesis due to their powerful generative capabilities. However, their high computational complexity and large sampling steps pose challenges for real-time applications. Fortunately, the Consistency Model (CM) provides a solution to greatly reduce the number of sampling steps from hundreds to a few, typically fewer than four, significantly accelerating the synthesis of diffusion models. However, applying CM to text-conditioned human motion synthesis in latent space yields unsatisfactory generation results. In this paper, we introduce \textbf{MotionPCM}, a phased consistency model-based approach designed to improve the quality and efficiency for real-time motion synthesis in latent space. Experimental results on the HumanML3D dataset show that our model achieves real-time inference at over 30 frames per second in a single sampling step while outperforming the previous state-of-the-art with a 38.9\% improvement in FID. The code will be available for reproduction.   
\end{abstract}    
\section{Introduction}
Driven by the advancement of multimodal approaches, human motion synthesis can accommodate different conditional inputs, including text~\cite{lin2018generating,zhang2022motiondiffuse}, action categories~\cite{tevet2023human,chen2023executing}, action sequence~\cite{dai2025motionlcm} and music~\cite{li2021ai}. These developments bring immense potential in various domains, such as the gaming industry, film production and virtual reality. 

MotionDiffuse~\cite{zhang2022motiondiffuse} is the first to apply the diffusion model to generate human motion, achieving remarkable performance. However, MotionDiffuse processes the entire motion sequence with the diffusion model, leading to high computational cost and long inference time. To alleviate these issues, MLD~\cite{chen2023executing} utilising a Variational Autoencoder (VAE)~\cite{kingma2013auto} to compress the motion sequence into latent codes before feeding them to the diffusion model. This approach greatly boosts both the speed per sampling step and quality of motion synthesis. However, it still requires 50 inference steps with the acceleration of DDIM~\cite{song2020denoising}, making it impractical to implement motion synthesis in real time. 

Building upon MLD, MotionLCM~\cite{dai2025motionlcm} utilise the Latent Consistency Model (LCM)~\cite{luo2023latent}, enabling few-step inference and thus achieving real-time motion synthesis with a diffusion model. However, LCM’s design suffers from several issues, including consistency issues caused by accumulated stochastic noise during multi-step sampling as shown in Figure~\ref{fig:lcm-vs-pcm}. In addition, LCM suffers from significantly degraded sample quality in low-step sampling or a large Classifier-free Guidance (CFG)~\cite{ho2022classifier} scale. Identifying these flaws of LCM, Phased Consistency Model (PCM)~\cite{wang2024phased} introduces a refined architecture to address these limitations. Taking Figure~\ref{fig:lcm-vs-pcm} as an example, PCM is trained with two sub-trajectories, enabling efficient 2-step deterministic sampling without introducing stochastic noise. 

% Experimental results on image generation tasks demonstrate that PCM outperforms LCM with different sampling steps. Nevertheless, the performance of PCM on human motion generation tasks remains unexplored. 

\begin{figure}[t!]
\centering
\includegraphics[width=0.4\textwidth]{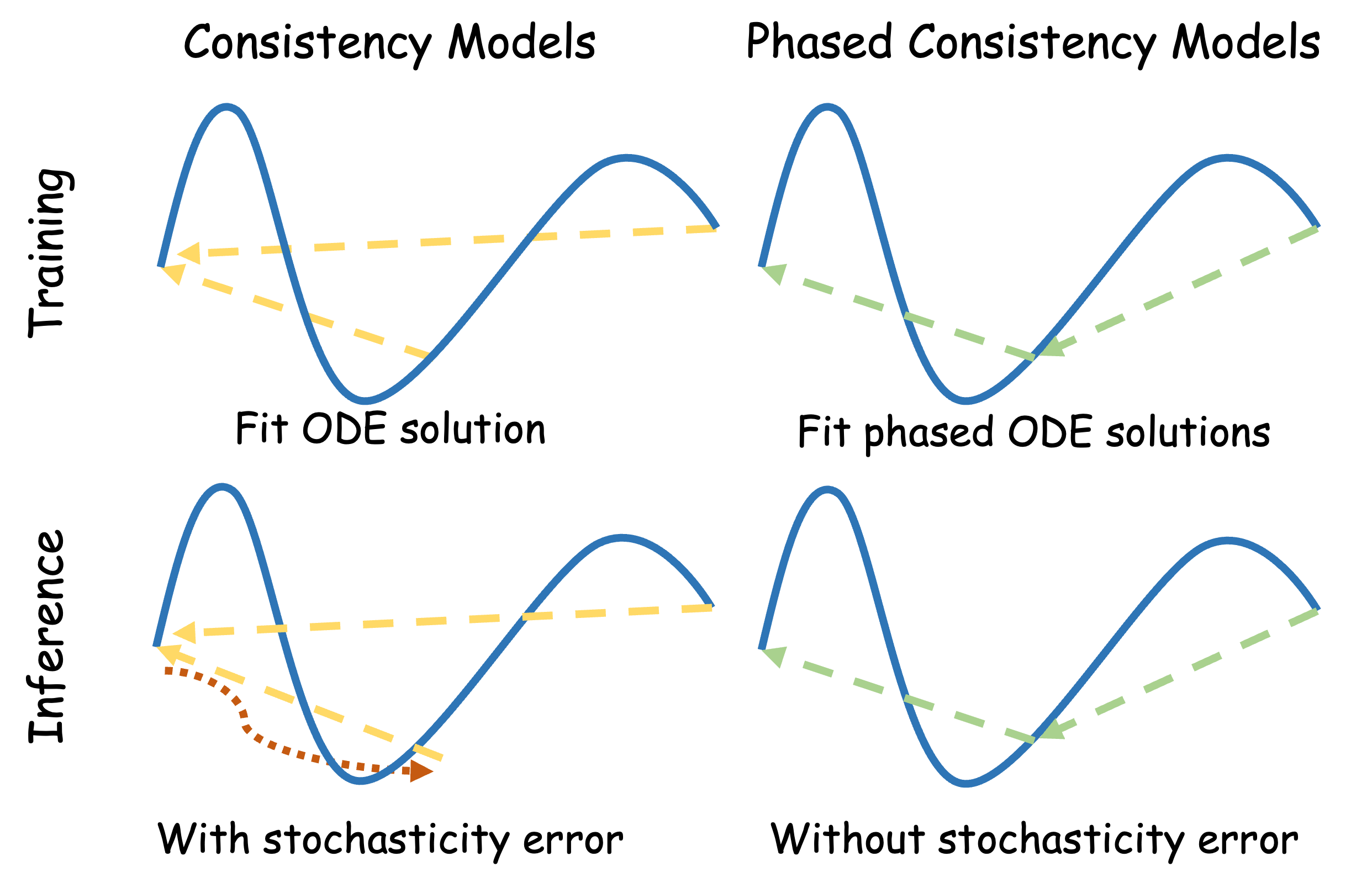}
\caption{Differences between Consistency/Latent Consistency Models and Phased Consistency Models in multi-step sampling.}
\label{fig:lcm-vs-pcm}
\end{figure}

In this paper, we incorporate PCM into the motion synthesis pipeline and propose a new motion synthesis approach, \textbf{MotionPCM}, allowing real-time motion synthesis with improved generation quality (see Figure~\ref{fig:performance}). Similar to MotionLCM, our model is distilled from MLD. However, unlike MotionLCM, we split the entire trajectory into \textit{M} segments where \textit{M} represents the number of inference steps. Instead of forcing all points to the trajectory’s starting point, each point is assigned to an interval, and consistency is enforced only within its respective interval by aligning it to the interval's start. This design allows us to achieve \textit{M}-step sampling deterministically. Furthermore, inspired by PCM and Consistency Trajectory Model (CTM)~\cite{kim2023consistency}, we employ an additional discriminator to enforce distribution consistency, leading to the enhanced performance for motion synthesis in low-step settings.

We summarise our main contributions as follows:
\begin{itemize}
    \item As the first to leverage the multi-interval design of PCM, we propose MotionPCM, an improved pipeline for real-time motion synthesis. 
    \item Introducing a well-designed discriminator to enhance distribution consistency significantly boost the quality of motion synthesis.
    \item Experiments on two widely used datasets demonstrate that our approach achieves state-of-the-art performance in terms of speed and generation quality, requiring fewer than four sampling steps. Additionally, in the CFG scale analysis, our model consistently outperforms the improved version of MotionLCM across various CFG scales and demonstrates significantly better robustness to scale variations.
\end{itemize}
\begin{figure}[t!]
\centering
\includegraphics[width=0.35\textwidth]{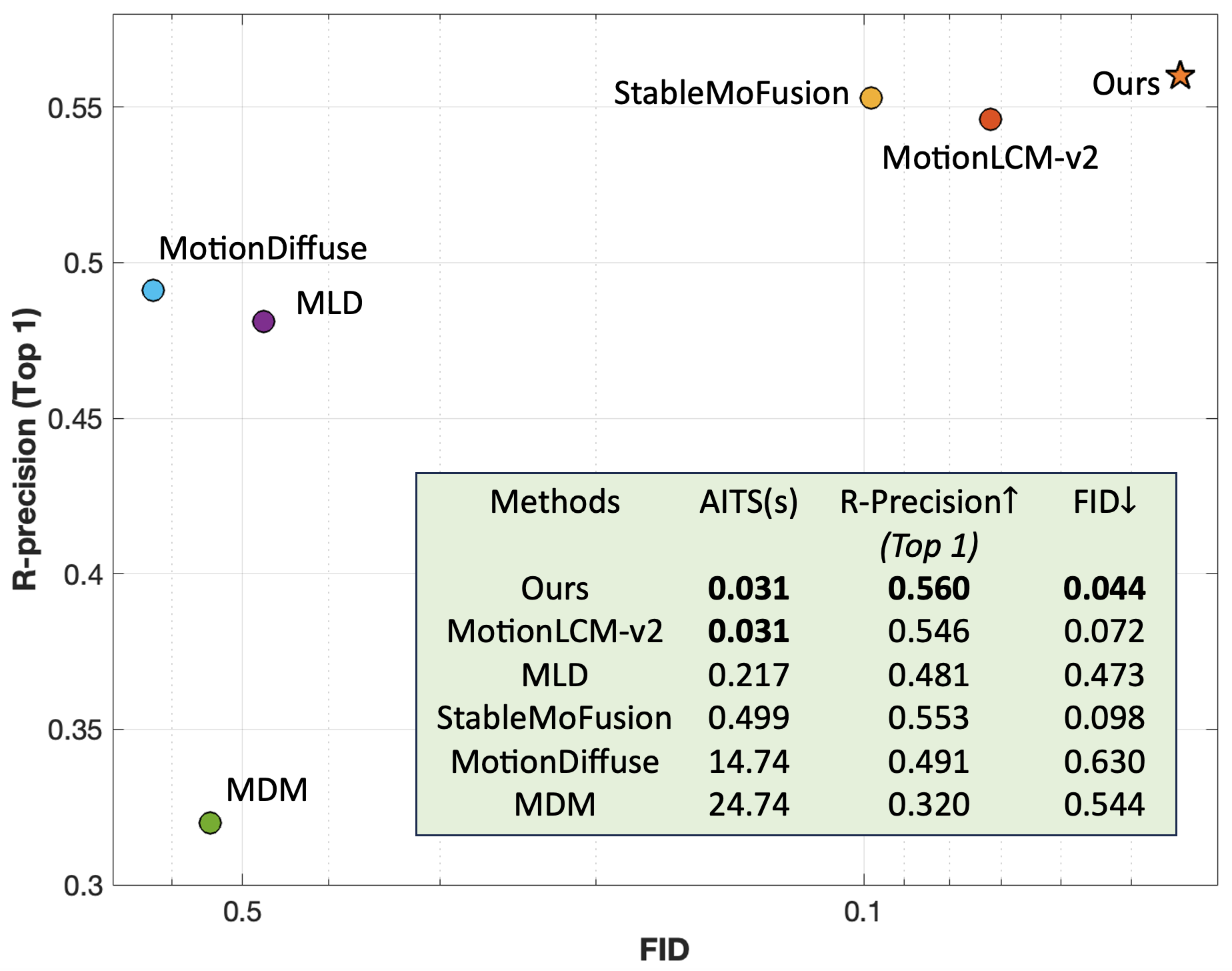}
\caption{Comparison of other motion synthesis methods with our method. AITS represents the time required to generate a motion sequence from a textual description. To facilitate display, the $x$-axis is plotted on a logarithmic scale.}
\label{fig:performance}
\end{figure}

\section{Related Work}
\subsection{Motion Synthesis}

Motion synthesis seeks to generate human motion under various conditions to support a wide range of applications~\cite{zhang2023generating,ghosh2021synthesis,guo2020action2motion, von2018recovering}. With advances in deep learning, the field has shifted towards deep generative models such as Variational Autoencoders (VAEs)~\cite{kingma2013auto} and Generative Adversarial Networks (GANs)~\cite{goodfellow2020generative}. More recently, diffusion models have further transformed motion synthesis by utilising noise-based iterative refinement to generate highly diverse and realistic human motion\cite{zhang2022motiondiffuse,tevet2023human,ho2020denoising,song2020denoising}.

Among the first works to adopt diffusion models for motion synthesis is MotionDiffuse~\cite{zhang2022motiondiffuse} which demonstrates that simply applying DDPM~\cite{ho2020denoising} to the raw motion sequence can outperform prior GAN-based~\cite{lee2019dancing} or VAE-based~\cite{guo2022generating} motion synthesis approaches. However, raw motion sequences are often noisy and redundant, leaving challenges for diffusion models to learn robust correlations between prior and data distributions~\cite{chen2023executing}. To address these issues, MLD~\cite{chen2023executing} integrates a transformer-based VAE with a long skip connection~\cite{ronneberger2015u} to produce representative low-dimensional latent codes from the raw sequences and perform the diffusion process in this latent space. This design improves generation quality whilst greatly reducing computational overhead. Taking MLD as a teacher network for distillation, MotionLCM~\cite{dai2025motionlcm} further reduces the sampling steps from more than 50 steps to a few (less than four) by using CM~\cite{song2023consistency,luo2023latent}. Subsequently, MotionLCM‐v2~\cite{motionlcm-v2} further refines the VAE network in MLD, yielding additional performance gains.

\subsection{Acceleration of Diffusion Models}
Since the speed is a bottleneck in diffusion models, various techniques have been proposed to accelerate them. A popular one is DDIM~\cite{song2020denoising}, which transforms the Markov chains of DDPM into a non-Markov process, thereby enabling skip sampling to accelerate generation. Subsequently, the Consistency Model (CM)~\cite{song2023consistency} advances this further by imposing a consistency constraint that maps noisy inputs directly to clean outputs without iterative denoising. This enables single‐step generation and substantially enhances speed. Latent Consistency Model (LCM)~\cite{luo2023latent} builds on CMs by operating in a latent space, unlike CMs, which work in the pixel domain. This approach enables LCM to handle more challenging tasks, such as text-to-image or text-to-video generation, with improved efficiency. As a result, LCM can serve as a foundational component to accelerate latent diffusion models, which has been employed in motion synthesis domain such as MotionLCM~\cite{dai2025motionlcm}. However, it still faces limitations—particularly in balancing efficiency, consistency, and controllability with varying inference steps, leaving room for further refinement.  Analysing the reasons behind these challenges of LCM, Phased Consistency Model (PCM)~\cite{wang2024phased}) partitions the ODE path into multiple sub-paths and enforces consistency within each sub-path. Additionally, PCM incorporates an adversarial loss to address the low-quality issue of image generation in low-sampling steps. In this paper, we apply PCM to the domain of motion synthesis and propose a new method, \textbf{MotionPCM}, for generating high-quality motion sequences in real time.

\section{Preliminaries}
\label{preliminaries}
Given $x(0) \sim p_0$, the data distribution, a traceable diffusion process w.r.t discrete time $t$ defined by $\alpha_tx_0+\sigma_t\epsilon$ is normally used to transform $x(0)$ to $x(T) \sim p_T$, a prior distribution. Likewise, score-based diffusion models~\cite{song2020score} define a continuous Stochastic Differential Equation (SDE) as the diffusion process:
\begin{equation}
    dx_t = f(x_t,t)dt + g(t)dw_t,
\end{equation}
where $(w_t)_{t \in [0, T]}$ is the standard $d$-dimensional Wiener process, $f: \mathbb{R}^d \times \mathbb{R}^{+}\rightarrow \mathbb{R}^d $ is a $\mathbb{R}^d$-valued function and $g: \mathbb{R}^{+} \rightarrow \mathbb{R}$ is a scalar function. The reverse-time SDE transforms the prior distribution back to the original data distribution. It is expressed as:
\begin{equation}
    dx_t = [f(x_t,t)-g(t)^2\nabla_x\log p_t(x_t)]dt + g(t)d\Bar{w}_t,
\end{equation}
where $\Bar{w}$ is again the standard $d$-dimensional Wiener process in the reversed process, $p_t(x_t)$ represents the probability density function of $x_t$ at time $t$. To estimate the score $\nabla_x \log p_t(x)$, a score-based model $s_\theta(x,t)$ is trained to approximate $\nabla_x \log p_t(x)$ as much as possible. 

There exists a deterministic reversed-time trajectory~\cite{song2020score}, satisfying an ODE, known as the probability flow ODE (PF-ODE): 
\begin{equation}
dx_t = [f(x_t,t)- \frac{1}{2}g(t)^2\nabla_x\log p_t(x_t)]dt.   
\end{equation}

Rather than using $s_\theta(x,t)$ to predict the score, consistency models~\cite{song2023consistency} directly learn a function $f_\theta(\cdot, t)$ to predict the solution of PF-ODE by mapping any points in the ODE trajectory to the origin of this trajectory, $x_\epsilon$, where $\epsilon$ is a fixed small positive number. Formally, for all $t, t' \in [ \epsilon, T] $, it holds that:
\begin{equation}
    f_\theta(x_t,t) = f_\theta(x_{t'}, t') = x_\epsilon.
\end{equation}

However, for multi-step sampling, CM will introduce random noise at each step since generating intermediate states along the sampling trajectory involves reintroducing noise, which accumulates and causes inconsistencies in the final output. To address this issue, PCM~\cite{wang2024phased} splits the solution trajectory of PF-ODE into multiple sub-intervals with \textit{M}+1 edge timesteps $s_0$, $s_1$, ..., $s_M$, where $s_0=\epsilon$ and $s_M=T$. Each sub-trajectory is treated as an independent CM, with a consistency function $f^m(\cdot,\cdot)$ defined as: for all $t, t' \in [s_m, s_{m+1}]$,
\begin{equation}
    f^m(x_t,t)= f^m(x_t',t') = x_{s_m}.
\label{eq5}
\end{equation}

\begin{figure*}[t]
    \centering
    \includegraphics[width=0.9 \linewidth]{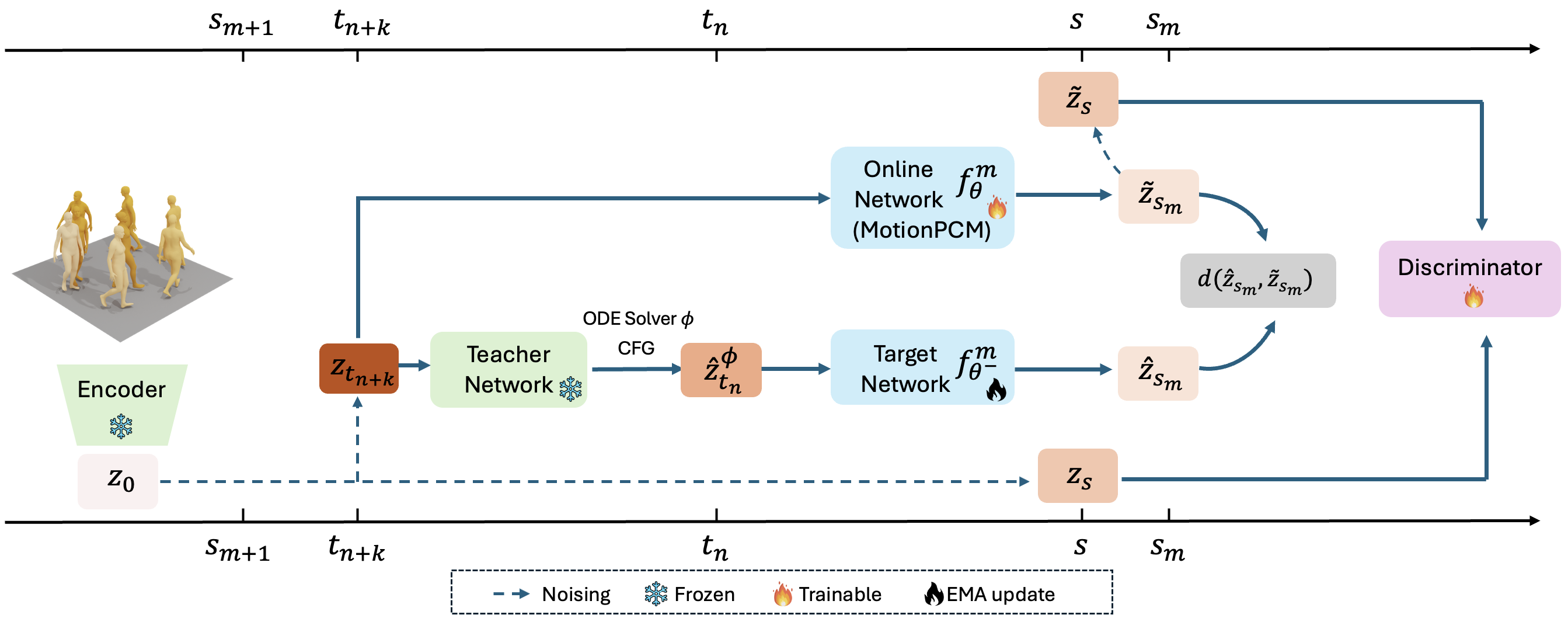}
    \caption{The pipeline of our proposed MotionPCM. In the training phase, a pre-trained VAE encodes the motion sequence to a latent code $z_0$,  which goes $n+k$ diffusion steps to produce $z_{t_{n+k}}$. $z_t{_{n+k}}$ is denoised to $\hat{z}_{t_n}$ through a teacher network and an ODE solver. $\hat{z}_{t_n}$ is passed through a target network  to predict $\hat{z}_{s_m}$. Simultaneously, $z_{t_{n+k}}$ is denoised to $\tilde{z}{_{s_m}}$ through the online network directly. A consistency loss within the time interval $[s_m,s_{m+1}]$ is applied by comparing $\tilde{z}{_{s_m}}$ and $\hat{z}_{s_m}$. Additionally, adversarial training is performed by introducing different noises to $\tilde{z}_{s_m}$ and $z_0$, generating  $\tilde{z}_s$ and $z_s$ respectively. These are then compared through a discriminator to enforce realism and improve model performance. The trainable components include the online network and the discriminator, whereas the encoder and teacher networks remain frozen during training. The target network is updated using the exponential moving average.}
    \label{pipeline}
\end{figure*}

%In \cite{lu2022dpm}, an exact solution to PF-ODE from time $t$ to $s$ is:
In \cite{lu2022dpm}, the explicit transition formula for the PF-ODE solution from time $t$ to 
$s$ is given by:
\begin{equation}
    x_s = \frac{\alpha_s}{\alpha_t} x_t + \alpha_s \int_{\lambda_t}^{\lambda_s} e^{-\lambda} \sigma_{\tau(\lambda)} \nabla \log P_{\tau(\lambda)}(x_{\tau(\lambda)}) \, d\lambda,
\end{equation}
where $\lambda_t = \ln{\frac{\alpha_t}{\sigma_t}}$ and $\tau$ is the inverse function of $t \mapsto \lambda_t$. Using an epsilon (noise) prediction network $\epsilon_\theta(x_t,t)$, this solution can be approximated as: 
\begin{equation}
    x_s = \frac{\alpha_s}{\alpha_t} x_t - \alpha_s \int_{\lambda_t}^{\lambda_s} e^{-\lambda} \epsilon_\theta(x_{\tau(\lambda)}, \tau(\lambda)) \, d\lambda.
\end{equation}
The solution needs to know the noise predictions throughout the entire interval between time $s$ and $t$, while consistency models can only access $x_t$ for a single inference. To address this, \cite{wang2024phased} parameterises $F_\theta(x,t,s)$ as follows: 
\begin{equation}
   F_\theta(x,t,s) =x_s = \frac{\alpha_s}{\alpha_t} x_t - \alpha_s \hat{\epsilon}_\theta(x_t, t) \int_{\lambda_t}^{\lambda_s} e^{-\lambda} \, d\lambda, 
\label{F_theta}
\end{equation}

Eq.~\eqref{F_theta} shares the same format as DDIM (see proof in the supplementary material). To satisfy the boundary condition of each sub-trajectory of the PCM, i.e., $f^m(x_{s_m},s_m) = x_{s_m}$, a parameterised form $f_\theta^m$ below is typically employed: 
\begin{equation}
    f^m_\theta(x_t,t) = c_\text{skip}^m(t)x_t + c_\text{out}^m(t)F_\theta(x_t,t,s_m),
\end{equation}
where $c_\text{skip}^m(t)$ gradually increases to 1 and $c_\text{out}^m(t)$ progressively decays to 0 as $t$ decreases over the time interval from $s_{m+1}$  to $s_m$. In fact, in Eq.~\eqref{F_theta}, the boundary condition $F_\theta(x_{s_m},s_m,s_m)= \frac{\alpha_{s_m}}{\alpha_{s_m}}x_{s_m}-0 = x_{s_m}$ is inherently satisfied. Hence, it can be reduced as below for direct use:%as a result, the simplified form shown below can be used directly:
\begin{equation}
    f^m_\theta(x,t)=F_\theta(x,t,s_m).
\label{eq9}
\end{equation}

\section{Method}
As shown in Figure~\ref{pipeline}, we introduce MotionPCM, a novel framework for real-time motion synthesis. To provide a clear understanding of our method, we divide our method into four key components. Section~\ref{VAE}
explains the use of a Variational Autoencoder (VAE) and latent diffusion model as pre-training models to initialise the framework. Section~\ref{motionpcm} describes the integration of the phased consistency model within the motion synthesis pipeline. Section~\ref{disc} details the design and role of the discriminator, which provides adversarial loss to enforce distribution consistency while improving overall performance. Section~\ref{inference} illustrates the process of generating motion sequences from a prior distribution using PCM during inference.

\subsection{VAE and Latent Diffusion Model for Motion Data Pre-training}
\label{VAE}

Following MLD~\cite{chen2023executing}, we first employ a transformer-based VAE~\cite{kingma2013auto} to compress motion sequences into a lower-dimensional latent space. More specifically, the encoder maps motion sequences $x \in \mathbb{R}^{L\times d}$, where $L$ is the frame length and $d$ is the number of features, to a latent code $z_0 = \mathcal{E}(x) \in \mathbb{R}^{N\times d'} $ where $N$ and $d'$ are much smaller than $L$ and $d$. Then a decoder is used to reconstruct the motion sequence $\hat{x}=\mathcal{D}(z_0)$. This process significantly reduces the dimensionality of the motion data, accelerating latent diffusion training in the second stage.  Building upon this, MotionLCM-V2~\cite{motionlcm-v2} introduces an improved VAE to enhance the representation quality of motion data. In our work, we adopt the improved VAE proposed by MotionLCM-V2 as the backbone. 

In the second stage, we train a latent diffusion model in the latent space learnt by the enhanced VAE, following MLD. Here, the latent diffusion model is an epsilon prediction network. Readers are referred to \cite{chen2023executing} for more details. This trained diffusion model will be used as the teacher network to guide the distillation process in our work.

\subsection{Accelerating Motion Synthesis via PCM}
\label{motionpcm}
\textbf{Definition.} Following the definition of PCM~\cite{wang2024phased}, we split our solution trajectory $z$ in the latent space into $M$ sub-trajectories, with edge timestep $\{s_m\mid m=0,1,2,\cdots,M\}$ where $s_0 = \epsilon$ and $s_M = T$. In each sub-time interval $[s_m, s_{m+1}]$, the consistency function $f^m$ is defined as Eq.~\eqref{eq5}. We train $f_\theta^m$ in Eq.~\eqref{eq9} to estimate $f^m$, applying consistency constraint on each sub-trajectory, i.e., $f_\theta^m(z_t,t)=f_\theta^m(z_t', t')=z_{s_m}$ for all $t, t' \in [s_m, s_{m+1}] $ and $m\in [0, M) \cap \mathbb{Z}$.
% Given text condition $c$, in each sub-trajectory $[s_m, s_{m+1}]$, we extend the consistency function $f^m$ defined in Eq.~\ref{eq5} as  $f^m(z_t,t,c)=z_{s_m}$ for $t \in [s_m, s_{m+1}]$.

\textbf{PCM Consistency distillation.} 
Once we obtain the pre-trained VAEs and latent diffusion model from Section~\ref{VAE}, we extract the representative latent code $z_0$ from VAE and use this latent diffusion model as our frozen teacher network to distill our MotionPCM model, i.e., online network in Figure~\ref{pipeline}. Following~\cite{dai2025motionlcm}, online network ($f_\theta$) is initialised from the teacher network with trainable weights $\theta$, while the target network ($f_{\theta^-}$) is also initialised from the teacher network but updated using Exponential Moving Average (EMA) of the online network's parameters. We obtain $z_{t_{n+k}}$ by applying the forward diffusion with $n+k$ steps to $z_0$, positioning it within the time interval $[s_m, s_{m+1}]$.

This work focuses on text-conditioned motion synthesis, where Classifier-free Guidance (CFG)~\cite{ho2022classifier} is frequently used to effectively align the generative results of diffusion models with the text conditions. Following previous works~\cite{chen2023executing,dai2025motionlcm,luo2023latent}, we also employ CFG in our framework. To distinguish $\hat\epsilon_\theta$ used in the consistency model, we use $\tilde\epsilon_\theta$ to represent the diffusion model, which corresponds to our teacher network. It can be expressed as:   
\begin{equation}
\begin{split}
    \tilde\epsilon(z_{t_{n+k}}, t_{n+k}, \omega, c) = &\ (1 + \omega)\tilde\epsilon(z_{t_{n+k}}, t_{n+k}, c) \\
    &- \omega\tilde\epsilon(z_{t_{n+k}}, t_{n+k}, \emptyset)
\end{split}
\end{equation}
where $c$ denotes text condition, the guidance scale $\omega$ is uniformly sampled from $[\omega_{min},\omega_{max}]$, and $\emptyset$ indicates an empty condition (i.e., a blank text input). $\hat{z}_{t_n}^\phi$ is then estimated from $z_{t_{n+k}}$ by performing k-step skip using $\tilde\epsilon(z_t{_{n+k}},t_{n+k},\omega,c)$, followed by an ODE solver $\phi$, such as DDIM~\cite{song2020denoising}. To efficiently perform the guided distillation, \cite{dai2025motionlcm,luo2023latent} add $\omega$ into an augmented consistency function $f_\theta(z_t,t,\omega,c) \mapsto z_0$. Similarly, in our work, we extend our phased consistency function to $f^m_\theta(z_t,t,\omega,c) \mapsto z_{s_m}$.   

Following CMs~\cite{song2023consistency,wang2024phased}, the phased consistency distillation loss is then defined as follows:
\begin{equation}
    \mathcal{L}_{PCD}(\theta,\theta^-) = \mathbb{E}[d(f_\theta^m(z_{t_{n+k}},t_{n+k},\omega,c),f_{\theta^-}^m(z_{t_n}^\phi,t_n,\omega,c))]
\end{equation}
where $d$ is Huber loss~\cite{huber1992robust} in our implementation. The online network's parameters $\theta$ are updated by minimising $\mathcal{L}_{PCD}$ through the standard gradient descent algorithms, such as AdamW~\cite{loshchilov2017decoupled}. Meanwhile, as mentioned earlier, the target network's parameters, $\theta^-$ is updated in EMA fashion: $\theta^-=\mu \theta^- + (1-\mu)\theta$.  

\subsection{Discriminator}
\label{disc}

Inspired by the work of~\cite{kim2023consistency,wang2024phased}, which demonstrated that incorporating an adversarial loss from a discriminator can enhance the image generation quality of diffusion models in few-step sampling settings, we also integrate an additional discriminator into our motion synthesis pipeline. 

As illustrated in Figure~\ref{pipeline}, a pair of $\tilde{z}_s$ and $z_s$ is sent to the discriminator. Specifically, we compute the solutions $\tilde{z}_{s_m} = f_\theta^m(z_{t_{n+k}}, t_{n+k}, \omega, c)$ and $\hat{z}_{s_m} = f_{\theta^-}^m(z_{t_n}^\phi, t_n, \omega, c)$ using the online and target networks, respectively. Following~\cite{wang2024phased}, noise is further added to $\tilde{z}_{s_m}$, generating $\tilde{z}_s$ for $s \in [s_m, s_{m+1}]$. 

However, unlike~\cite{wang2024phased}, which derives $z_s$ from $\hat{z}_{s_m}$, we obtain $z_s$ directly from $z_0$. This modification aligns with \cite{kim2023consistency}, which emphasises that leveraging direct training signals from data labels is key to achieving optimal performance.

\begin{figure}[h]
    \centering
    \includegraphics[width=0.95 \linewidth]{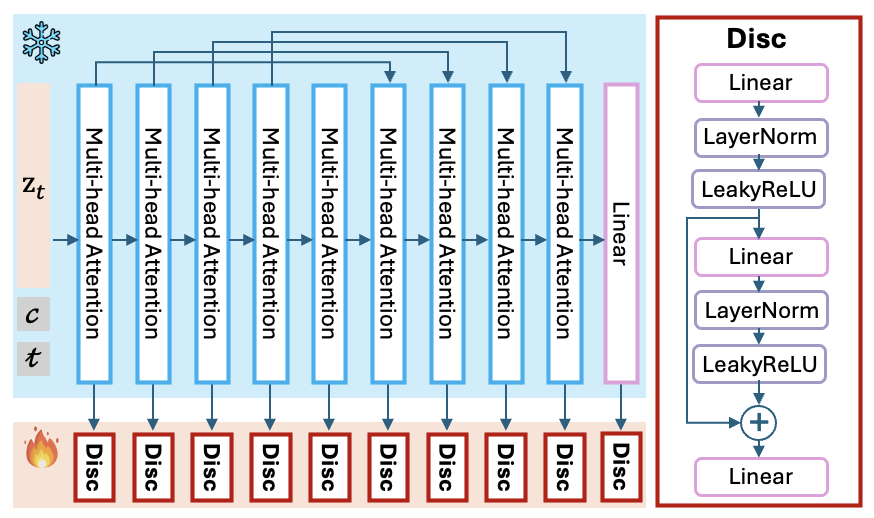}
    \caption{Detailed structure of discriminator in our proposed MotionPCM model.}
    \label{fig:disc}
\end{figure}

In Figure~\ref{fig:disc}, we illustrate the detailed structure of the discriminator in our proposed MotionPCM model. For convenience, we use \( z_t \), \( c \), and \( t \) to denote the input to the discriminator, which, in practice, can be either \( \tilde{z}_s \) or \( z_s \). The discriminator comprises two components: a frozen teacher network, consisting of multiple multi-head attention layers, and a trainable discrimination module, which includes linear layers, layer normalisation, activation functions, and residual connections.

During training, inspired by~\cite{wang2024phased}, we apply an adversarial loss to the output of the discriminator as follows:
\begin{equation}
    \mathcal{L}_{adv} = \text{ReLU}(1+f_D(z_s, s,c)) + \text{ReLU}(1-f_D(\tilde{z}_s,s,c))
\end{equation}
where $f_D$ is the discriminator, ReLU is an non-linear activation function. The model is trained with $\mathcal{L}_{adv}$ using a min-max strategy~\cite{goodfellow2020generative}.

The joint loss combining phased consistency distillation loss and adversarial loss is expressed as:
\begin{equation}
    \mathcal{L}_{all} = \mathcal{L}_{PCD} + \lambda\mathcal{L}_{adv}
\end{equation}
where $\lambda$ is a hyper-parameter. 

\subsection{Inference} 
\label{inference}
During inference, we sample $z_T$  from a prior distribution, such as standard normal distribution $\mathcal{N}(0,1)$. Based on the transition map defined in~\cite{wang2024phased}: $ f_{m,m'}(x_t,t) = f^{m'} \big( \cdots f^{m-2} \big( f^{m-1} \big( f^m(\mathbf{x}_t, t), s_m \big), s_{m-1} \big) \cdots, s_{m'} \big)$ which transform any point $x_t$ on $m$-th sub-trajectory to the solution point of $m'$-th trajectory, we can get the solution estimation $\hat{z}_0=f_{M-1,0} (x_T,T)$. Finally, the human motion sequence $\hat{x}$ is generated through the decoder  $\mathcal{D}(\hat{z}_0)$. 
\begin{table*}[h!]
\centering
\scriptsize
\setlength{\tabcolsep}{4pt} % Adjust column spacing
\renewcommand{\arraystretch}{1.2} % Adjust row spacing
\begin{tabular}{@{}llcccccccc@{}}
\toprule
\multirow{2}*{\textbf{Datasets}} &
\multirow{2}*{\textbf{Methods}} &
\multirow{2}*{\textbf{AITS $\downarrow$}} &
\multicolumn{3}{c}{\textbf{R-Precision $\uparrow$}} &
\multirow{2}*{\textbf{FID $\downarrow$}} &
\multirow{2}*{\textbf{MM Dist $\downarrow$}} &
\multirow{2}*{\textbf{Diversity $\rightarrow$}} &
\multirow{2}*{\textbf{MModality $\uparrow$}} \\ 
\cmidrule(lr){4-6}
& & & \textbf{Top 1} & \textbf{Top 2} & \textbf{Top 3} & & & & \\ 
\midrule
\multirow{16}*{\textbf{HumanML3D}} & 
Real & - & $0.511^{\pm.003}$ & $0.703^{\pm.003}$ & $0.797^{\pm.002}$ & $0.002^{\pm.000}$ & $2.974^{\pm.008}$ & $9.503^{\pm.065}$ & - \\
\cmidrule(lr){2-10}

&TM2T~\cite{guo2022tm2t} & $0.760$ & $0.424^{\pm.003}$ & $0.618^{\pm.003}$ & $0.729^{\pm.002}$ & $1.501^{\pm.017}$ & $3.467^{\pm.011}$ & $8.589^{\pm.076}$ & ${\underline{2.424}}^{\pm.093}$ \\
&MotionDiffuse~\cite{zhang2022motiondiffuse} & $14.74$ & $0.491^{\pm.001}$ & $0.681^{\pm.001}$ & $0.782^{\pm.001}$ & $0.630^{\pm.011}$ & $3.113^{\pm.001}$ & $9.410^{\pm.049}$ & $1.553^{\pm.072}$ \\
&MDM~\cite{tevet2023human} & $24.74$ & $0.320^{\pm.005}$ & $0.498^{\pm.004}$ & $0.611^{\pm.007}$ & $0.544^{\pm.044}$ & $5.556^{\pm.027}$ & ${{9.559}}^{\pm.086}$ & $\textbf{2.799}^{\pm.072}$ \\
&MLD~\cite{chen2023executing} & $0.217$ & $0.481^{\pm.003}$ & $0.673^{\pm.003}$ & $0.772^{\pm.002}$ & $0.473^{\pm.013}$ & $3.196^{\pm.010}$ & $9.724^{\pm.082}$ & $2.413^{\pm.079}$ \\
&T2M-GPT~\cite{zhang2023generating} & $0.380$ & $0.492^{\pm.003}$ & $0.679^{\pm.002}$ & $0.775^{\pm.002}$ & $0.141^{\pm.005}$ & $3.121^{\pm.009}$ & $9.722^{\pm.082}$ & $1.831^{\pm.048}$ \\
&ReMoDiffuse~\cite{zhang2023remodiffuse} & $0.624$ & $0.510^{\pm.005}$ & $0.698^{\pm.006}$ & $0.795^{\pm.004}$ & $0.103^{\pm.004}$ & $2.974^{\pm.016}$ & $9.018^{\pm.075}$ & $1.795^{\pm.043}$ \\
&StableMoFusion~\cite{huang2024stablemofusion} & $0.499$ & $0.553^{\pm.003}$ & ${0.748}^{\pm.002}$ & ${{0.841}}^{\pm.002}$ & $0.098^{\pm.003}$ & $2.770^{\pm.006}$ & $9.748^{\pm.092}$ & $1.774^{\pm.051}$ \\
&B2A-HDM~\cite{xie2024towards} & - & $0.511^{\pm.002}$ & $0.699^{\pm.002}$ & $0.791^{\pm.002}$ & $0.084^{\pm.004}$ & $3.020^{\pm.010}$ & $\textbf{9.526}^{\pm.080}$ & $1.914^{\pm.078}$ \\
&MotionLCM~\cite{dai2025motionlcm} (4-step) & $0.043$ & $0.502^{\pm.003}$ & $0.698^{\pm.002}$ & $0.798^{\pm.002}$ & $0.304^{\pm.012}$ & $3.012^{\pm.007}$ & $9.607^{\pm.066}$ & $2.259^{\pm.092}$ \\
&MotionLCM-V2~\cite{motionlcm-v2} (1-step) & $\textbf{0.031}$ & $0.546^{\pm.003}$ & $0.743^{\pm.002}$ & $0.837^{\pm.002}$ & $0.072^{\pm.003}$ & ${2.767}^{\pm.007}$ & $9.577^{\pm.070}$ & $1.858^{\pm.056}$ \\
&MotionLCM-V2~\cite{motionlcm-v2} (2-step) & $0.038$ & $0.551^{\pm.003}$ & $0.745^{\pm.002}$ & $0.836^{\pm.002}$ & ${0.049}^{\pm.003}$ & ${2.765}^{\pm.008}$ & $9.584^{\pm.066}$ & $1.833^{\pm.052}$ \\
&MotionLCM-V2~\cite{motionlcm-v2} (4-step) & $0.050$ & ${0.553}^{\pm.003}$ & $0.746^{\pm.002}$ & $0.837^{\pm.002}$ & $0.056^{\pm.003}$ & $2.773^{\pm.009}$ & $9.598^{\pm.067}$ & $1.758^{\pm.056}$ \\
\cmidrule(lr){2-10}
&Ours (1-step) & \textbf{0.031} &  $\textbf{0.560}^{\pm.002}$ &$\textbf{0.752}^{\pm.003}$ & $\textbf{0.844}^{\pm.002}$& ${0.044}^{\pm.003}$ &$\textbf{2.711}^{\pm.008}$ & $\underline{9.559}^{\pm.081}$& $1.772^{\pm.067}$\\
&Ours (2-step) & 0.036 &  ${0.555}^{\pm.002}$ &${0.749}^{\pm.002}$ & ${0.839}^{\pm.002}$& $\underline{0.033}^{\pm.002}$ &${2.739}^{\pm.007}$ & $9.618^{\pm.088}$& $1.760^{\pm.068}$\\
&Ours (4-step) & 0.045 & $\underline{0.559}^{\pm.003}$ & $\textbf{0.752}^{\pm.003}$ & $\underline{0.842}^{\pm.002}$ & $\textbf{0.030}^{\pm.002}$ & $\underline{2.716}^{\pm.008}$ & ${9.575}^{\pm.082}$ & $1.714^{\pm.062}$ \\
\toprule

\multirow{11}*{\textbf{KIT-ML}} & 
Real & - & $0.424^{\pm.005}$ & $0.649^{\pm.006}$ & $0.649^{\pm.006}$ & $0.031^{\pm.004}$ & $2.788^{\pm.012}$ & $ 11.08^{\pm.097}$ & - \\
\cmidrule(lr){2-10}

&TM2T~\cite{guo2022tm2t} & $0.760$ & $ 0.280^{\pm.005}$ & $0.463^{\pm.006}$ & $0.587^{\pm.005}$ & $3.599^{\pm.153}$ & $4.591^{\pm.026}$ & - & $\textbf{3.292}^{\pm.081}$ \\
&MotionDiffuse~\cite{zhang2022motiondiffuse} & $14.74$ & $0.417^{\pm.004}$ & $0.621^{\pm.004}$ & $0.739^{\pm.004}$ & $1.954^{\pm.062}$ & $2.958^{\pm.005}$ & $\textbf{11.10}^{\pm.143}$ & $0.730^{\pm.013}$ \\
&MDM~\cite{tevet2023human} & $24.74$ & $0.164^{\pm.004}$ & $0.291^{\pm.004}$ & $0.396^{\pm.004}$ & $0.497^{\pm.021}$ & $9.191^{\pm.022}$ & $10.847^{\pm.109}$ & $1.907^{\pm.214}$ \\
&MLD~\cite{chen2023executing} & $0.217$ & $0.390^{\pm.008}$ & $0.609^{\pm.008}$ & $0.734^{\pm.007}$ & $0.404^{\pm.027}$ & $3.204^{\pm.027}$ & $10.80^{\pm.117}$ & $\underline{2.192}^{\pm.071}$ \\
&T2M-GPT~\cite{zhang2023generating} & $0.380$ & $ 0.416^{\pm.006}$ & $0.627^{\pm.006}$ & $0.745^{\pm.006}$ & $0.514^{\pm.029}$ & $3.007^{\pm.023}$ & ${10.921}^{\pm.108}$ & $1.570^{\pm.039}$ \\
&ReMoDiffuse~\cite{zhang2023remodiffuse} & $0.624$ & $0.427^{\pm.014}$ & $0.641^{\pm.004}$ & $0.765^{\pm.055}$ & $\textbf{0.155}^{\pm.006}$ & $\textbf{2.814}^{\pm.012}$ & $10.80^{\pm.105}$ & $1.239^{\pm.028}$ \\
&StableMoFusion~\cite{huang2024stablemofusion} & $0.499$ & $\textbf{0.445}^{\pm.006}$ & ${0.660}^{\pm.005}$ & ${0.782}^{\pm.004}$ & $\underline{0.258}^{\pm.029}$ &  - & $\underline{10.936}^{\pm.077}$ & $1.362^{\pm.062}$ \\
&B2A-HDM~\cite{xie2024towards} & - & $0.436^{\pm.006}$ & $0.653^{\pm.006}$ & $0.773^{\pm.005}$ & $0.367^{\pm.020}$ & $2.946^{\pm.024}$ & ${10.86}^{\pm.124}$ & $1.291^{\pm.047}$ \\
\cmidrule(lr){2-10}
&Ours (1-step)& $\textbf{0.031}$ & $0.433^{\pm.007}$& $0.654^{\pm.007}$ & $0.781^{\pm.008}$ & $0.355^{\pm.011}$& $\underline{2.820}^{\pm.022}$ & $10.788^{\pm.078}$ & $1.337^{\pm.047}$\\
&Ours (2-step)& $\underline{0.036}$ & ${0.437}^{\pm.005}$& $\textbf{0.664}^{\pm.005}$ & $\underline{0.787}^{\pm.006}$ & $0.294^{\pm.011}$ & ${2.844}^{\pm.018}$ & $10.827^{\pm.094}$ & $1.254^{\pm.050}$\\
&Ours (4-step)& $0.045$ & $\underline{0.443}^{\pm.005}$& $\textbf{0.664}^{\pm.004}$ & $\textbf{0.789}^{\pm.005}$ & $0.336^{\pm.013}$& $2.881^{\pm.023}$ & $10.758^{\pm.096}$ & $1.258^{\pm.056}$\\
\bottomrule

\end{tabular}
\caption{Performance comparison of various methods across multiple metrics on HumanML3D and KIT-ML dataset. The best results are in \textbf{bold}, and the second best results are \underline{underlined}. $\downarrow$ means the lower is better while $\uparrow$ means the higher is better. $\rightarrow$ represents the closer to the value of Real is better. Results on the KIT-ML dataset are unavailable for the MotionLCM series.}
\label{tab:results}
\end{table*}

\section{Numerical Experiments}

\textbf{Datasets.}
We base our experiments on the widely used \textbf{HumanML3D} dataset~\cite{guo2022generating} and \textbf{KIT Motion-Language (KIT-ML)} dataset~\cite{Plappert2016KIT}. HumanML3D comprises 14,616 distinct human motion sequences accompanied by 44,970 textual annotations, whereas the KIT-ML dataset contains 6,353 textual annotations and 3,911 motions. In line with previous studies~\cite{dai2025motionlcm,chen2023executing,guo2022generating}, we employ a redundant motion representation that includes root velocity, root height, local joint positions, velocities, root-space rotations, and foot-contact binary indicators to ensure a fair comparison.

\textbf{Evaluation metrics.}
Following~\cite{guo2022generating,chen2023executing}, we evaluate our model using the following metrics: (1) \textbf{Average Inference Time per Sentence (AITS)}, which measures the time required to generate a motion sequence from a textual description, with lower values indicating faster inference; (2) \textbf{R-Precision}, capturing how accurately generated motions match their text prompts by checking whether the top-ranked motions align with the given descriptions, where higher scores indicate better accuracy; (3) \textbf{Frechet Inception Distance (FID)}, assessing how closely the distribution of generated motions resembles real data, where lower scores indicate better quality; (4) \textbf{Multimodal Distance (MM Dist)}, quantifying how well the motion features align with text features, with lower values signalling a tighter match; (5) \textbf{Diversity}, which calculates variance through motion features to indicate the variety of generated motions across different samples; and (6) \textbf{MultiModality (MModality)}, measuring generation diversity conditioned on the same text by evaluating how many distinct yet valid motions can be produced for a single prompt.

\begin{figure*}[t!]
\centering
\resizebox{0.8\textwidth}{!}{%
\renewcommand{\arraystretch}{1.2} % Adjust row height for better readability
\begin{tabular}{cccc}
\hline
\parbox[c][1cm][c]{0.2\linewidth}{\centering \textbf{Real}} & 
\parbox[c][1cm][c]{0.2\linewidth}{\centering \textbf{MotionPCM (Ours)}} & 
\parbox[c][1cm][c]{0.25\linewidth}{\centering \textbf{MotionLCM-V2}} & 
\parbox[c][1cm][c]{0.15\linewidth}{\centering \textbf{MLD}} \\ \hline

\multicolumn{4}{c}{\small{\textit{The rigs {\color{red}walk forward}, then {\color{red}turn around}, and {\color{red}continue walking} before stopping where he started.}}} \\

\includegraphics[height=2.5cm,keepaspectratio]{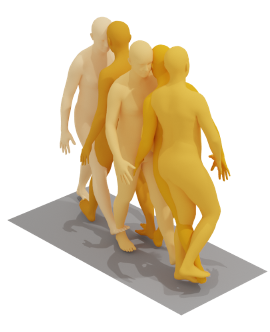} &
  \includegraphics[height=2.5cm,keepaspectratio]{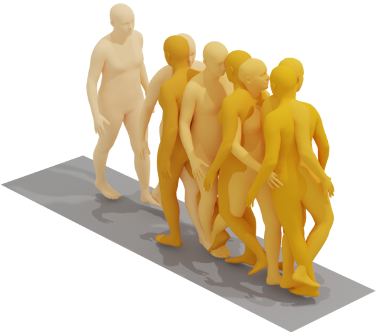} &
  \includegraphics[height=2.5cm,keepaspectratio]{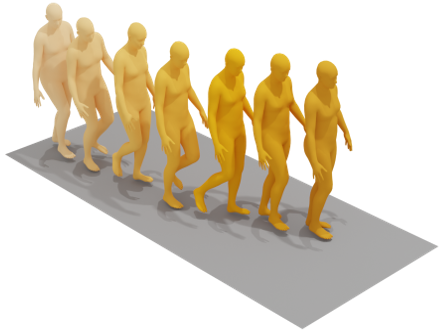} &
  \includegraphics[height=2.5cm,keepaspectratio]{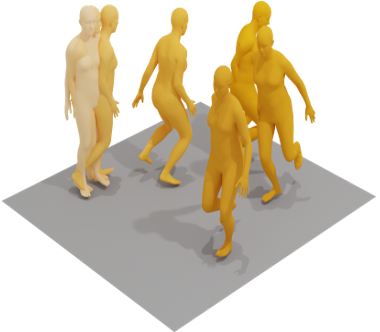} \\ 
\hline

\multicolumn{4}{c}{\small{\textit{A person {\color{red}walks} with a limp leg.}}} \\
\includegraphics[height=2.5cm,keepaspectratio]{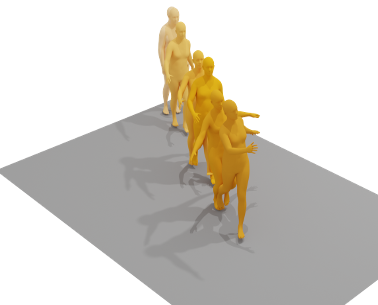} &
  \includegraphics[height=2.5cm,keepaspectratio]{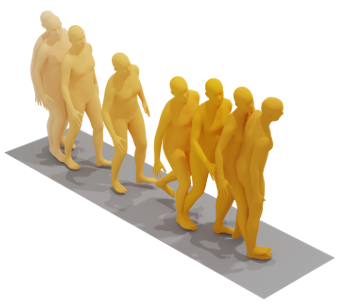} &
  \includegraphics[height=2.5cm,keepaspectratio]{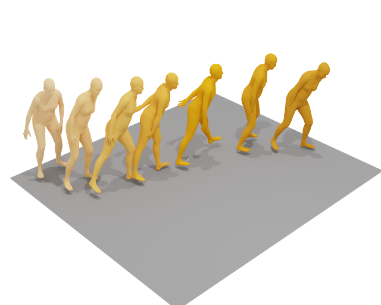} &
  \includegraphics[height=2.5cm,keepaspectratio]{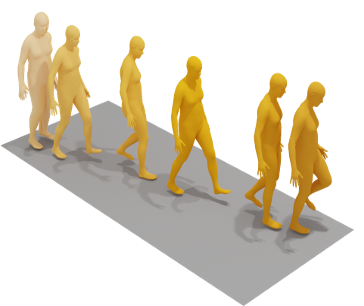} \\ 
\hline

\multicolumn{4}{c}{\small{\textit{A man {\color{red}walks forward} a few steps, {\color{red}raise his left hand} to his face, then {\color{red}continue walking in a circle}.}}} \\
\includegraphics[height=2.5cm,keepaspectratio]{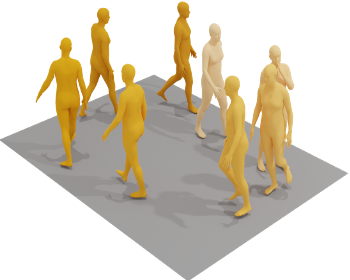} &
  \includegraphics[height=2.5cm,keepaspectratio]{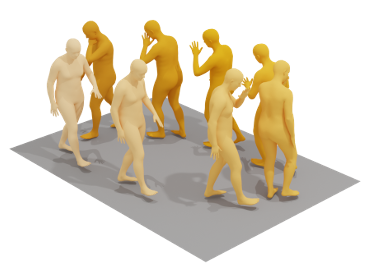} &
  \includegraphics[height=2.5cm,keepaspectratio]{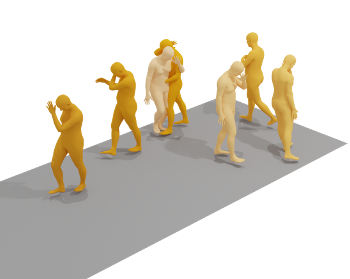} &
  \includegraphics[height=2.5cm,keepaspectratio]{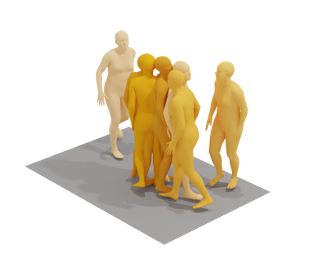} \\ 
\hline

\end{tabular}
}
\caption{Qualitative comparison of motion synthesis methods. Lighter colours represent earlier time points.}
\label{fig:visual_comparison}
\end{figure*}

\textbf{Implementation details.} We conduct our experiments on a single NVIDIA RTX 6000 GPU with a batch size of 128 motion sequences fed into our model. The model is trained over 384K iterations using an AdamW~\cite{loshchilov2017decoupled} optimiser with parameters $\beta_1= 0.9$, $\beta_2=0.999$ and an initial learning rate $2e-4$, which follows a cosine decay schedule. For the loss function, we set $\lambda = 0.1$. The CFG scale samples between $w_{min} =5$, and $w_{max}=15$ in the training phase and $\omega = 14$ is used in the test phase for all experiments. The Exponential Moving Average (EMA) rate is set as $\mu =0.95$. Additionally, we use DDIM~\cite{song2020denoising} as our ODE solver with a skip step $k =100$. Sentence-T5~\cite{ni2021sentence} is used to encode the text condition. Following~\cite{motionlcm-v2}, the latent code is compressed to a $\mathbb{R}^{16 \times 32}$ representation.

\subsection{Text-to-motion synthesis}

In this section, we evaluate the performance of our proposed MotionPCM method on the text-to-motion task. Following~\cite{dai2025motionlcm,guo2022generating}, we conduct each experiment 20 times to establish the results within a confidence interval 95\% on both HumanML3D and KIT-ML datasets.   

On the large-scale \textbf{HumanML3D} dataset, we thoroughly compare MotionLCM-v2~\cite{motionlcm-v2}, the improved version of MotionLCM~\cite{dai2025motionlcm}, with our method in different sampling steps, as it is the closest competitor to our method. As illustrated in Table~\ref{tab:results}, our method achieves a sampling speed comparable to MotionLCM-v2 (AITS) while significantly outperforms it in R-Precision Top-1, Top-2 and Top-3, FID, and MM Dist across different sampling steps. Furthermore, our 1-step and 4-step variants also outperform the counterparts of MotionLCM-V2 in terms of Diversity. These results demonstrate that our model is more superior than MotionLCM-V2 although both support real-time inference.  

Compared to other approaches on \textbf{HumanML3D} dataset, the inference speed of our method with 1-step sampling (over 30 frames per second) surpasses all alternatives by a large margin, demonstrating its time efficiency. Regarding R-Precision, our 1-step variant achieves the highest accuracy across Top-1, Top-2 and Top-3 metrics compared to other approaches. This consistent improvement highlights the reliability of MotionPCM in accurately aligning generated motions with textual descriptions. Similarly, in FID and MM Distance metrics, our method achieves the best scores with its 4-step and 1-step variants respectively. These results further highlight MotionPCM's capability to generate high-quality and semantically consistent motions. Although our method does not achieve the best scores in Diversity compared to \cite{xie2024towards}, our 1-step variant ranks second. 

For the \textbf{KIT-ML} dataset, our model still achieves the best or second best performance on Top 1, Top 2 and Top 3 of R-Precision with our 2-step and 4-step variants, underscoring its strong text-motion alignment. Although the FID score of our 2-step variant ranks third compared to \cite{huang2024stablemofusion,zhang2023remodiffuse}, our inference speed is around 14 times and 17 times faster than theirs, respectively. Moreover, our 2-step variant ranks second in MM Dist while falling short of achieving state-of-the-art performance in terms of Diversity and MModality.     

\begin{figure}[t]
    \centering
    \includegraphics[width=0.85 \linewidth]{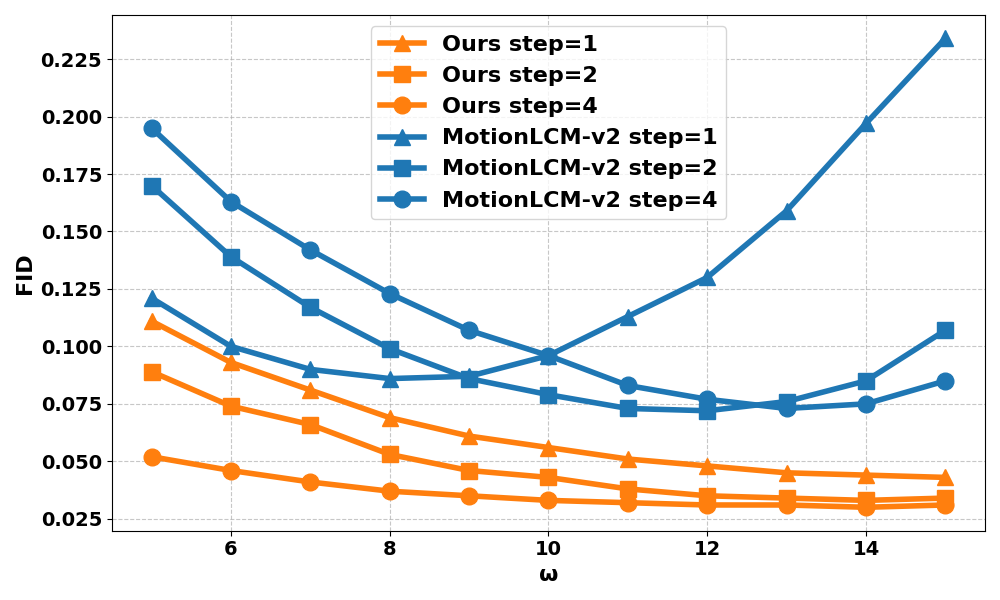}
    \caption{Comparison of the impact of $\omega$ on Our Method and MotionLCM-V2~\cite{motionlcm-v2} on HumanML3D dataset under different sampling steps in test phase.}
    \label{fig:omega}
\end{figure}

Figure~\ref{fig:visual_comparison} demonstrates a qualitative comparison of motion generation between our MotionPCM model, MotionLCM-v2~\cite{motionlcm-v2}, and MLD~\cite{chen2023executing}. Each row features a prompt alongside the motions generated by the three models, clearly demonstrating the advantages of our method over the others. For the first prompt, “The rigs walk forward, then turn around, and continue walking before stopping where he started,” MotionLCM-v2 does not perform the turning action. MLD moves to the left initially and fails to return to the starting position when turning back. In contrast, our method accurately executes the entire sequence as described. In the case of the second prompt, “a person walks with a limp leg,” our model produces a motion that reasonably reflects the limp. MotionLCM-v2 exaggerates the limp, while MLD fails to depict the limp entirely. For the third prompt, “a man walks forward a few steps, raises his left hand to his face, then continues walking in a circle,” our model effectively completes the described actions. MotionLCM-v2 does not generate the circular walking pattern, and MLD produces an indistinct circle without the hand-raising gesture. These findings emphasise the ability of MotionPCM to generate detailed and accurate motion sequences. Overall, these qualitative and quantitative results on both large-scale and small-scale datasets demonstrate that MotionPCM is capable of producing high-quality motions in real time while maintaining superior alignment with the given textual descriptions, outperforming existing benchmarks in motion synthesis.
% , aligning closely with the prompt

\subsection{Investigation of CFG scale}

As our model and MotionLCM-V2~\cite{motionlcm-v2} both embed the CFG scale $\omega$ in the motion synthesis model, it is important to see how the CFG scale affects the model performance. It is clear from Figure~\ref{fig:omega} that our model performs much better than MotionLCM-V2 for each $\omega \in [5,15]$ across different sampling steps. Besides, our model is significantly more robust with different CFG scales. E.g. the FID fluctuation of our model with 1 step is within $[0.043, 0.111]$ while MotionLCM-V2 is $[0.086, 0.234]$. Furthermore, keeping similar with the findings from~\cite{wang2024phased}, our model has a more predictable pattern since as the CFG scale increases, the FID metrics of our method normally improve. In contrast, MotionLCM-V2 has no such patterns, and the best performance of their model for $\omega$ is unpredictable and $\omega$ may differ a lot to achieve the best performance at different sampling steps. Visual comparisons between two methods under different $\omega$ can be seen in the supplementary material.

\subsection{Ablation studies}

\begin{table}[t!]
\centering
\scriptsize
\setlength{\tabcolsep}{4pt} % Adjust column spacing
\renewcommand{\arraystretch}{1.2} % Adjust row spacing
\begin{tabular}{@{}lccccc@{}}
\toprule
\textbf{Methods} &
{\textbf{R-Precision Top 1$\uparrow$}} &
\textbf{FID $\downarrow$}&
\textbf{MM Dist $\downarrow$} \\ 
\midrule
MotionPCM $k=100$& ${0.560}^{\pm.002}$ & ${0.044}^{\pm.003}$ &${2.711}^{\pm.008}$ \\
MotionPCM $k=50$& ${ 0.554}^{\pm.003}$ & ${0.068}^{\pm.004}$ &${2.739}^{\pm.007}$ \\
MotionPCM $k=20$& ${0.547}^{\pm.003}$ & ${0.079}^{\pm.004}$ &${2.776}^{\pm.005}$ \\
MotionPCM $k=1$& ${0.531}^{\pm.002}$ & ${0.092}^{\pm.004}$ &${2.837}^{\pm.008}$ \\
\midrule
MotionPCM $\mu=0.95$ & ${0.560}^{\pm.002}$ & ${0.044}^{\pm.003}$ &${2.711}^{\pm.008}$ \\
MotionPCM $\mu=0.5$ & ${0.550}^{\pm.003}$ & ${0.039}^{\pm.003}$ &${2.774}^{\pm.005}$ \\
MotionPCM $\mu=0$ & ${0.559}^{\pm.003}$ & ${0.049}^{\pm.004}$ &${2.718}^{\pm.007}$ \\
\midrule
MotionPCM $\lambda=1$ & ${0.543}^{\pm.003}$ & ${0.038}^{\pm.002}$ &${ 2.793}^{\pm.008}$ \\
MotionPCM $\lambda=0.1$ & ${0.560}^{\pm.002}$ & ${0.044}^{\pm.003}$ &${2.711}^{\pm.008}$ \\
MotionPCM $\lambda=0$ & $0.547^{\pm.003}$ &$0.101^{\pm.006}$& $2.785^{\pm.008}$\\
\bottomrule
\end{tabular}
\caption{Ablation studies for single-step sampling on HumanML3D dataset.}
\label{tab:ablation}
\end{table}

We conduct ablation studies on our single-step sampling, as shown in Table~\ref{tab:ablation} to demonstrate the effectiveness of our network design. If we gradually reduce skip step $k$ from 100 to 1, all evaluation metrics degrade progressively. This exception may be due to the increased number of time points, which makes network fitting more difficult, as a smaller skip step $k$ results in more time points within each interval. 
% requiring more training iterations to achieve proper convergence. 

In addition, if we replace the target network with the online network directly—equivalent to setting $\mu=0$ in the exponential moving average (EMA) update, a slightly worse performance is observed compared to training using an EMA way with $\mu=0.95$. When setting $\mu=0.5$ for the EMA update, we observe a slight improvement in FID while R-Precision Top 1 and MM Dist performance decline.  

Last not least, removing a discriminator and its adversarial loss—equivalent to setting $\lambda=0$, leads to much worse performance, indicating the importance of the discriminator in enhancing model performance, particularly in low-step sampling scenarios. Compared to $\lambda=0.1$ used in our paper, setting $\lambda=1$ can improve FID while reducing R-precision Top 1 and MM Dist. Visual comparisons between methods trained with and without a discriminator can be found in the supplementary material.

\section{Conclusion}
In this paper, we present \textbf{MotionPCM}, a novel motion synthesis method that enables real-time motion generation while maintaining high quality and outperforming other state-of-the-art methods. By incorporating phased consistency into our pipeline, we achieve deterministic sampling without the accumulated random noise in multi-step sampling. Furthermore, the introduction of a well-designed discriminator dramatically improves sampling quality. Compared to our main competitor (MotionLCM-v2), MotionPCM demonstrates significantly better performance and robustness to different CFG scales, making it a more effective solution for real-time motion synthesis. 

\textbf{Acknowledgements.} 
LJ and HN are supported by the EPSRC [grant number
EP/S026347/1]. HN is also supported by The Alan Turing Institute under the EPSRC grant EP/N510129/1.

The authors thank Po-Yu Chen, Niels Cariou Kotlarek, François Buet-Golfouse, and especially Mingxuan Yi for their insightful discussions on diffusion models.

{
    \small
    \bibliographystyle{ieeenat_fullname}
    \bibliography{main}
}

\clearpage
\appendix
\section{Supplementary}
\subsection{Equivalent Parameterisation}
In this subsection, we show that Eq.~\eqref{F_theta} has the same format of DDIM, which is established in \cite{wang2024phased,lu2022dpm}. 
\begin{lemma}\label{Lemma1}
Let $F_\theta(x,t,s)$ be defined in  Eq.~\eqref{F_theta}, i.e.,
$$F_\theta(x,t,s) = \frac{\alpha_s}{\alpha_t}x_t-\alpha_s\,\hat{\epsilon}_\theta(x_t,t)
 \int_{\lambda_t}^{\lambda_s}e^{-\lambda}\,d\lambda.$$ 
  Then $F_\theta$ allows the equivalent representation: 
  \begin{eqnarray}\label{Eqn_PCM_DDIM}
  x_s = \frac{\alpha_s}{\alpha_t}\Bigl(x_t-\sigma_t\,\hat{\epsilon}_\theta(x_t,t)\Bigr)
 +\sigma_s\,\hat{\epsilon}_\theta(x_t,t).
 \end{eqnarray}
%where $\hat{\epsilon}_\theta$ is an estimator of $\frac{\int_{\lambda_t}^{\lambda_s} e^{-\lambda} \epsilon_\theta(x_{\tau(\lambda)}, \tau(\lambda)) \, d\lambda}{\int_{\lambda_t}^{\lambda_s} e^{-\lambda}d\lambda}$.
\end{lemma}

%\textbf{Lemma 1.} Following \cite{wang2024phased,lu2022dpm}, we show that Eq.~\eqref{F_theta}  $F_\theta(x,t,s)=\frac{\alpha_s}{\alpha_t}x_t-\alpha_s\,\hat{\epsilon}_\theta(x_t,t)
% \int_{\lambda_t}^{\lambda_s}e^{-\lambda}\,d\lambda$ has the same format with DDIM~\cite{song2020denoising}: $x_s = \frac{\alpha_s}{\alpha_t}\Bigl(x_t-\sigma_t\,\hat{\epsilon}_\theta(x_t,t)\Bigr)
%+\sigma_s\,\hat{\epsilon}_\theta(x_t,t).$ .

% i.e. $\frac{\alpha_s}{\alpha_t}x_t-\alpha_s\,\hat{\epsilon}_\theta(x_t,t)
% \int_{\lambda_t}^{\lambda_s}e^{-\lambda}\,d\lambda = \frac{\alpha_s}{\alpha_t}\Bigl(x_t-\sigma_t\,\hat{\epsilon}_\theta(x_t,t)\Bigr)
% +\sigma_s\,\hat{\epsilon}_\theta(x_t,t).$ 

%\textit{Proof.} 
\begin{proof}The proof follows \cite{wang2024phased,lu2022dpm}.  Eq.~\eqref{F_theta} is written as 
\begin{equation}
F_\theta(x,t,s)=x_s=\frac{\alpha_s}{\alpha_t}x_t-\alpha_s\,\hat{\epsilon}_\theta(x_t,t)
\int_{\lambda_t}^{\lambda_s}e^{-\lambda}\,d\lambda.
\end{equation}
Since 
\begin{equation}
\int_{\lambda_t}^{\lambda_s}e^{-\lambda}\,d\lambda=-e^{-\lambda}\Big|_{\lambda_t}^{\lambda_s}
=e^{-\lambda_t}-e^{-\lambda_s},
\end{equation}
and noting that $\lambda_t = \ln{\frac{\alpha_t}{\sigma_t}}$ and $\lambda_s = \ln{\frac{\alpha_s}{\sigma_s}}$, we have 
\begin{align}
\begin{split}
F_\theta(x,t,s)&=\frac{\alpha_s}{\alpha_t}x_t-\alpha_s\,\hat{\epsilon}_\theta(x_t,t)
(e^{-\lambda_{t}}-e^{-\lambda_{s}})\\
&=\frac{\alpha_s}{\alpha_t}x_t-e^{-\lambda_{t}}\alpha_s\,\hat{\epsilon}_\theta(x_t,t)+e^{-\lambda_{s}}\alpha_s\,\hat{\epsilon}_\theta(x_t,t)\\
&=\frac{\alpha_s}{\alpha_t}x_t-\frac{\sigma_t}{\alpha_t}\alpha_s\,\hat{\epsilon}_\theta(x_t,t)+\sigma_s\,\hat{\epsilon}_\theta(x_t,t)\\
&= \frac{\alpha_s}{\alpha_t}\Bigl(x_t-\sigma_t\,\hat{\epsilon}_\theta(x_t,t)\Bigr)
+\sigma_s\,\hat{\epsilon}_\theta(x_t,t).
\end{split}
\end{align}

\end{proof}

Recall that in DDIM, a typical update from time-step $t$ to $s$ is given by 
\begin{equation*}
    x_s=\frac{\sqrt{\bar{\alpha}_s}}{\sqrt{\bar{\alpha}_t}}\Bigl(x_t-\sqrt{1-\bar{\alpha}_t}\,\hat{\epsilon}_\theta(x_t,t)\Bigr)
+\sqrt{1-\bar{\alpha}_s}\,\hat{\epsilon}_\theta(x_t,t),
\end{equation*}
where $\bar{\alpha}_t$ is the cumulative noise-schedule term.  With $\alpha_t=\sqrt{\bar{\alpha}_t}$ and $\sigma_t=\sqrt{1-\bar{\alpha}_t}$, this update becomes 
\begin{equation}\label{Eqn_DDIM}
x_s=\frac{\alpha_s}{\alpha_t}\Bigl(x_t-\sigma_t\,\hat{\epsilon}_\theta(x_t,t)\Bigr)
+\sigma_s\,\hat{\epsilon}_\theta(x_t,t).
\end{equation}
By comparing Eq. \eqref{Eqn_DDIM} and Eq. \eqref{Eqn_PCM_DDIM}, we observe that Eq.~\eqref{F_theta} has the same format as DDIM. However, as noted in~\cite{wang2024phased}, the difference between Eq.~\eqref{F_theta} and DDIM lies in the meaning of $\hat{\epsilon}_\theta$. In DDIM, $\hat{\epsilon}_\theta$ refers to the first-order approximation of ODE, whereas in our model, the network $\hat{\epsilon}_\theta$ estimates $\frac{\int_{\lambda_t}^{\lambda_s} e^{-\lambda} \epsilon_\theta(x_{\tau(\lambda)}, \tau(\lambda)) \, d\lambda}{\int_{\lambda_t}^{\lambda_s} e^{-\lambda}d\lambda}$. 

\subsection{VAE performance}
In this section, we evaluate the performance of the VAE network, enhanced by~\cite{motionlcm-v2}, on the HumanML3D dataset, as presented in Table~\ref{tab:VAE_performance}.
\begin{table}[h]
    \centering
    \begin{tabular}{cccc}
        \toprule
       Latents &  FID $\downarrow$& MPJPE $\downarrow$ & Feature err. $\downarrow$ \\
       \midrule
       16$\times$ 32  &0.008 & 15.8 & 0.214   \\
         \bottomrule
         
    \end{tabular}
    \caption{VAE performance on MotionML3D dataset. MPJPE is measured in millimetre.}
    \label{tab:VAE_performance}
\end{table}

\subsection{Comparison with MLD* accelerating via DDIM}

\begin{figure}[ht]
\centering
\includegraphics[width=0.48\textwidth]{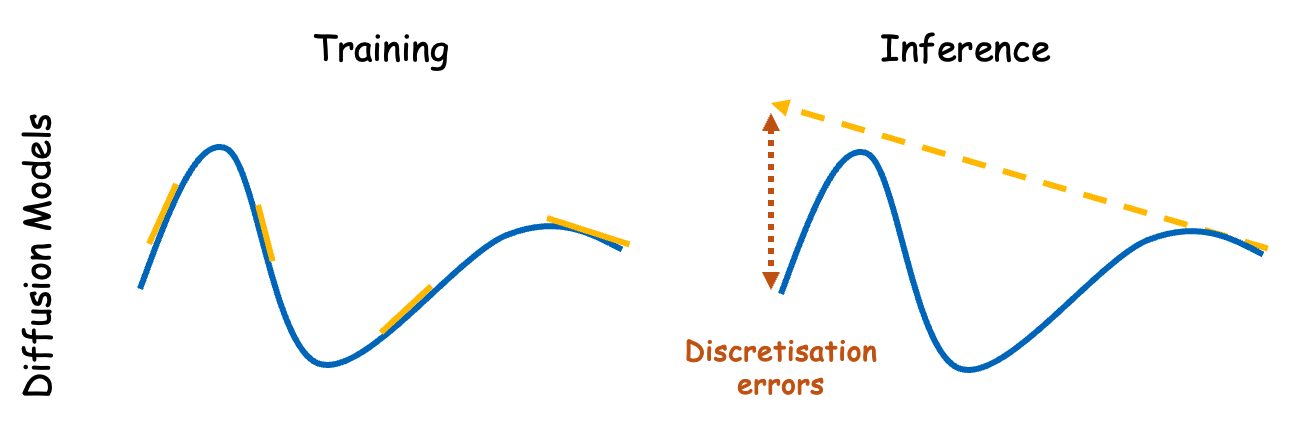}
\caption{Illustration of discretisation errors in the inference stage of diffusion models.}
\label{fig:discret}
\end{figure}

MLD* is the model obtained by retraining MLD~\protect\cite{chen2023executing} with the improved VAE proposed by MotionLCM-V2~\protect\cite{motionlcm-v2}, which is our baseline. In this section, we show results of accelerating MLD* using DDIM~\cite{song2020denoising} with different sampling steps. As we can see from Table~\ref{tab:DDIM_step}, directly applying DDIM to reduce the sampling steps results in significant performance degradation, with the FID of MLD* increasing from 0.065 to 24.34 as the sampling steps reduce from 50 to 1.  This degradation is attributed to the accumulation of discretisation errors (as shown in Figure~\ref{fig:discret}). In contrast, our model can accelerate MLD* sampling by reducing the number of steps from 50 to 1, while actually improving its performance. We attribute this improvement to our phased consistency model and discriminator design.

\begin{table*}[h]
\centering
\scriptsize
\setlength{\tabcolsep}{4pt} % Adjust column spacing
\renewcommand{\arraystretch}{1.2} % Adjust row spacing
\begin{tabular}{@{}lcccccc@{}}
\toprule
\multirow{2}*{\textbf{Methods}} &
\multicolumn{3}{c}{\textbf{R-Precision $\uparrow$}} &
\multirow{2}*{\textbf{FID $\downarrow$}} &
\multirow{2}*{\textbf{MM Dist $\downarrow$}} &
\multirow{2}*{\textbf{Diversity $\rightarrow$}} \\ 
\cmidrule(lr){2-4}
&  \textbf{Top 1} & \textbf{Top 2} & \textbf{Top 3} & & &  \\ 
\midrule
MLD* (50 steps)& ${0.545}^{\pm.002}$ &${0.739}^{\pm.003}$ & ${0.830}^{\pm.002}$& ${0.065}^{\pm.003}$ &${2.816}^{\pm.007}$ & $9.620^{\pm.098}$\\
\midrule
MLD* (10 steps) &  ${0.546}^{\pm.003}$ &${0.737}^{\pm.002}$ & ${0.830}^{\pm.002}$& ${0.070}^{\pm.003}$ &${2.816}^{\pm.008}$ & $9.689^{\pm.078}$\\
\midrule
MLD* (4 steps) & $0.515^{\pm.003}$ &$0.708^{\pm.003}$ & $0.807^{\pm.002}$&$0.133^{\pm.004}$& $2.976^{\pm.008}$& $9.574^{\pm.091}$ \\
\midrule
MLD* (2 steps) & $0.273^{\pm.002}$ &$0.428^{\pm.003}$ & $0.537^{\pm.003}$&$3.490^{\pm.034}$& $4.792^{\pm.011}$& $7.530^{\pm.076}$\\
\midrule
MLD* (1 step) & $0.035^{\pm.001}$ &$0.068^{\pm.002}$ & $0.101^{\pm.002}$&$24.34^{\pm.075}$& $7.913^{\pm.010}$& $4.299^{\pm.050}$\\
\midrule
Ours (1 step) &  $\textbf{0.560}^{\pm.002}$ &$\textbf{0.752}^{\pm.003}$ & $\textbf{0.844}^{\pm.002}$& $\textbf{0.044}^{\pm.003}$ &$\textbf{2.711}^{\pm.008}$ & $\textbf{9.559}^{\pm.081}$ \\ 
\bottomrule
\end{tabular}
\caption{The results of accelerating MLD* using DDIM across different sampling steps on HumanML3D dataset.}
\label{tab:DDIM_step}
\end{table*}

\subsection{Insights into the Importance and Design of the Discriminator}
\label{discriminator}

\begin{figure*}[h]
    \centering
    \includegraphics[width=0.9\linewidth]{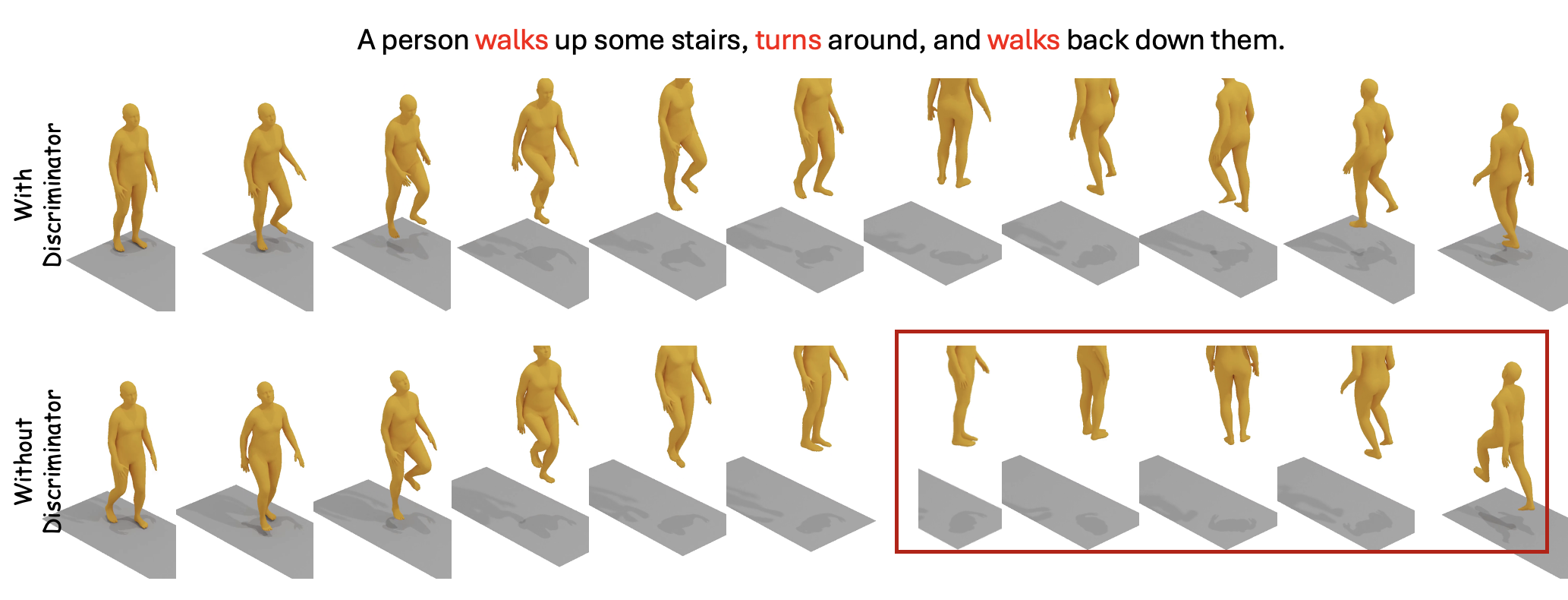}\\[1ex]
    \includegraphics[width=0.9\linewidth]{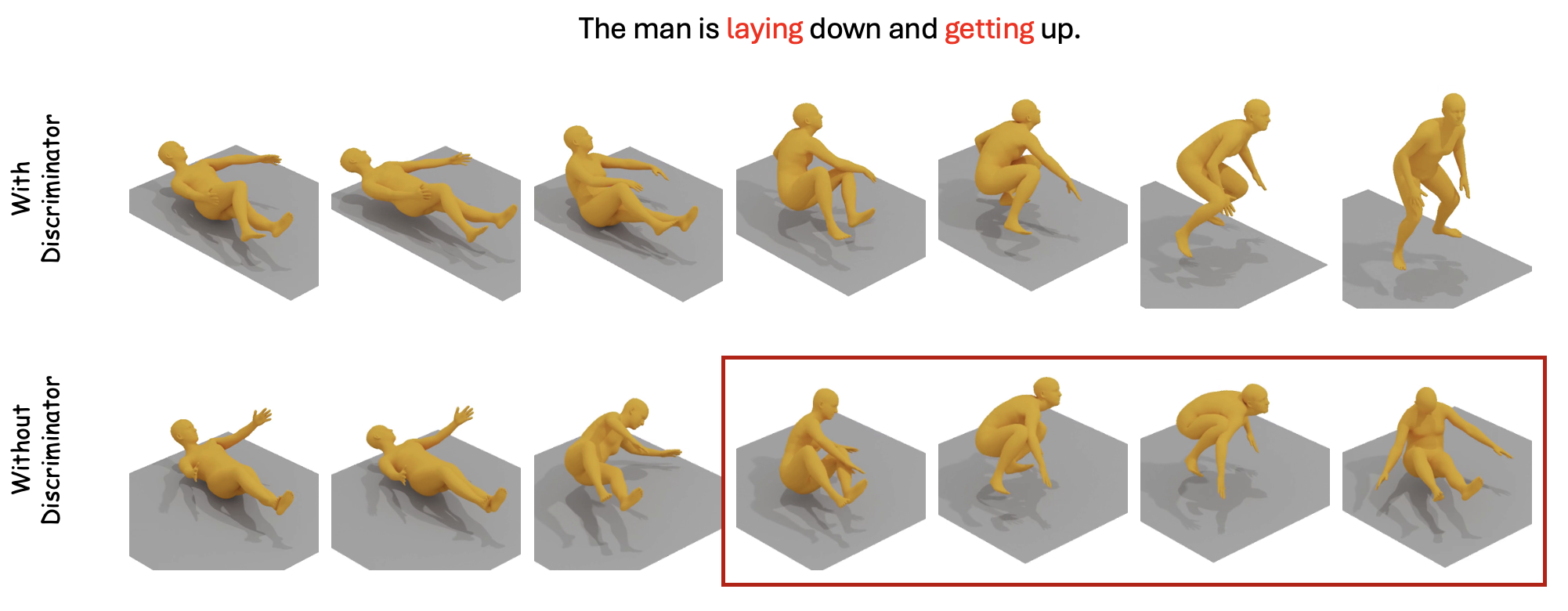}
    \caption{Ablation experiments on the role of the discriminator in our approach. The images are arranged from left to right, representing the motion over time. The low-quality motion generation is highlighted in the \textcolor{red}{red} box.}
    \label{fig:disc_compare}
\end{figure*}

To demonstrate the importance of the discriminator, we exhibit two visual comparisons with and without our discriminator in Figure~\ref{fig:disc_compare}. For the first given text condition 'A person walks up some stairs, turn around and walks down them', the model trained without the discriminator demonstrate some unexpected actions, such as no leg movements when the person turns around, violating the kinematic principles, and the person walking up the stairs again at the end of the video, which contradicts the textual description. Similarly, for the second example, the model trained without the discriminator shows the person lying down again after attempting to get up. In contrast, the model trained with the discriminator can generate smoother actions and better align with the textual condition.

In terms of the input for the discriminator, it is clear from Table~\ref{tab:VAE_performance} and Table~\ref{tab:DDIM_step} that the FID score of VAE is much lower (better) than the MLD* used to distill our model. As mentioned by CTM~\cite{kim2023consistency} that the direct training signal derived from the data label plays a crucial role in enabling a student model to even outperform its teacher model during the distillation phase. This is why we use the latent code of VAE to guide the discriminator training, instead of using the training signal from the teacher network.  On the other hand, if we follow the PCM approach~\cite{wang2024phased}, which feeds $\hat{z}_{s_m}$ from the teacher network to the discriminator, we can observe instability in the training and we cannot obtain meaningful results. Therefore, using $z_0$ can not only overcome the performance bottleneck imposed by MLD* but also stablise training, accelerating convergence.

\subsection{Guidance Scale Effect}
\label{cfg}

We also show the effect of the guidance scale $\omega$ on the FID score for the KIT-ML dataset in Figure~\ref{fig:omega-kit}. Similarly to the findings on the HumanML3D dataset, a larger $\omega$ usually leads to a better or comparable result, while the best performance on the KIT-ML dataset is achieved around $\omega = 12$ for all sampling steps. For fair comparison, we keep $\omega=14$ for all experiments in the test phase. MotionLCM-v2~\cite{motionlcm-v2} does not report their result on KIT-ML dataset, leaving no ways to compare.    

\begin{figure}[h]
    \centering
    \includegraphics[width=0.85 \linewidth]{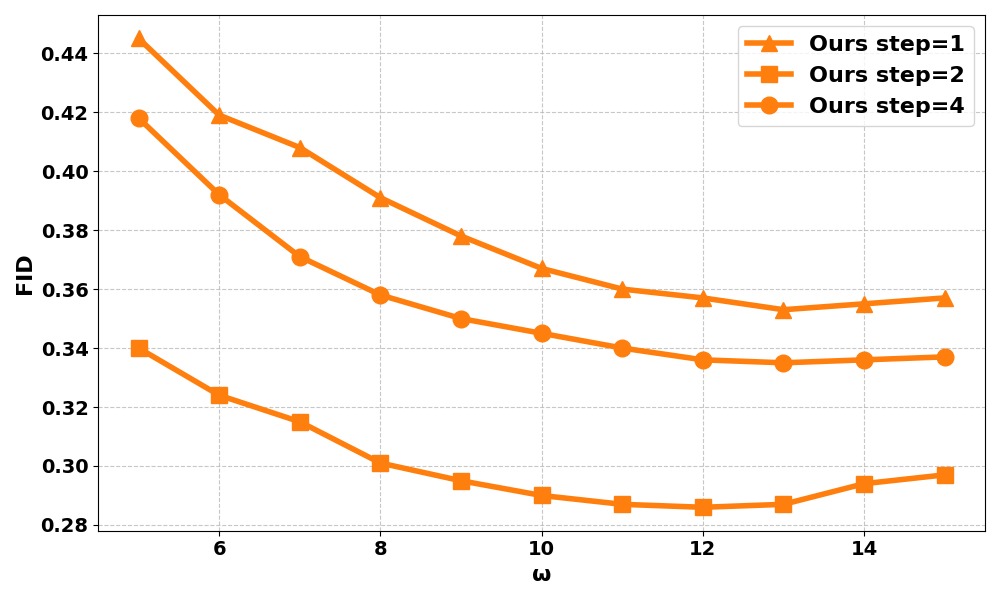}
    \caption{The impact of $\omega$ on Our Method for 1-step sampling.}
    \label{fig:omega-kit}
\end{figure}

\begin{figure*}[t]
    \centering
    \includegraphics[width=0.9\linewidth]{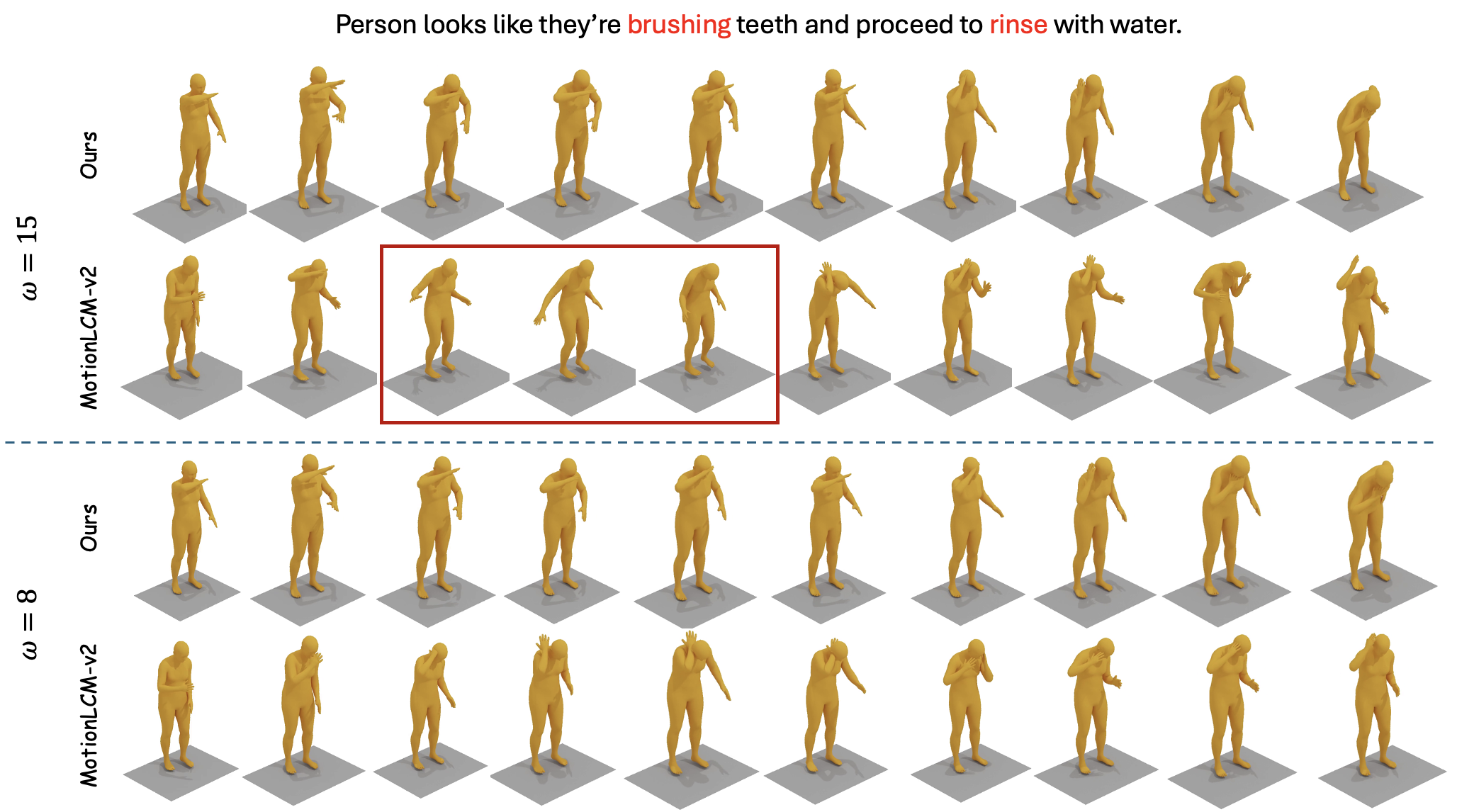}\\[1ex]
    \includegraphics[width=0.9\linewidth]{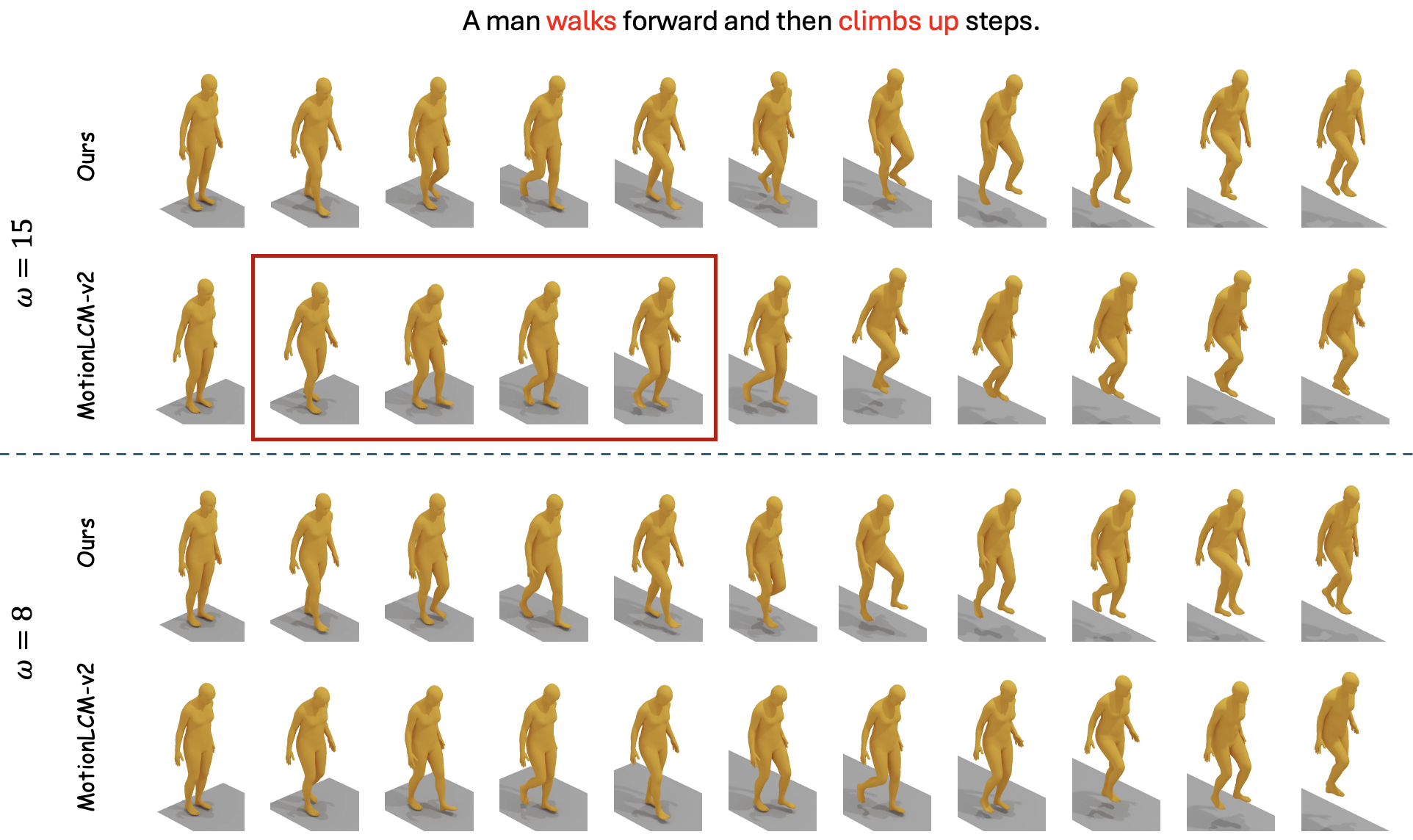}
    \caption{Qualitative demonstration of our method versus MotionLCM-v2 under different $\omega$. The images are arranged from left to right, representing the motion over time. The low-quality motion generation is highlighted in the \textcolor{red}{red} box.}
    \label{fig:supp-omega-sample}
\end{figure*}

Additionally, we show the visual comparison between our method and MotionLCM-v2 with different guidance scales $\omega$ in Figure~\ref{fig:supp-omega-sample}. It's clear to see that the performance of MotionLCM-v2 with $\omega=15$ is worse than $\omega=8$. For example, given the textual description `person looks like they're brushing teeth and proceed to rinse with water', MotionLCM-v2 on $\omega=15$ introduces several unrelated actions from frames 3 to 5 in this figure. Furthermore, given the textual description `A man walks forward and then climbs up steps', MotionLCM-v2 with $\omega=15$ exhibits a staggered gait during the initial walking phase, rendering it less natural. In contrast, our approach maintains consistency across different $\omega$. A clear comparison can be observed in the accompanying videos.

\subsection{Pseudocode}
To clearly and concisely illustrate our training process, we present the pseudocode for the training phase in Algorithm~\ref{algo}. 

\begin{algorithm*}[h]
    \caption{MotionPCM Training}
    \begin{algorithmic}[1]
        \State \textbf{Input:} motion dataset $\mathcal{M}$, PCM parameter $\theta$, discriminator $f_D$ with trainable parameters $\theta_D$, pre-trained frozen encoder $\mathcal{E}$, learning rate $\eta$, ODE solver $\Psi$, Huber loss $d$, EMA update rate $\mu$, noise schedule (drift coefficients $\alpha_t$, diffusion coefficients $\sigma_t$), CFG scale $[\omega_{\min}, \omega_{\max}]$, skip number of ODE solver $k$, discretized timesteps $\{t_i \mid i = 0, 1,2,\cdots,N\}$ where $t_0 = \epsilon$ and $t_N = T$, edge timesteps $\{s_m\mid m=0,1,2,\cdots,M\} \in \{t_i\}_{i=0}^{N}$ where $s_0=t_0$ and $s_M = t_N$.
        \State Training data: $\mathcal{M}_x = \{(\mathbf{x}, c)\}$
        \State $\theta^- \gets \theta$
        \Repeat
        \State Sample $(\mathbf{x}, c) \sim \mathcal{M}_x$, $n \sim \mathcal{U}(0, N - k)$ and $\omega \sim [\omega_{\min}, \omega_{\max}]$
        \State Obtain the latent code $z_0 = \mathcal{E}(x)$
        \State Sample $z_{t_{n+k}} \sim \mathcal{N}(\alpha_{t_{n+k}} z_0, \sigma^2_{t_{n+k}} I)$
        \State Determine $[s_m, s_{m+1}]$ given $n$
        \State $\mathbf{z}^{\phi}_{t_n} \gets (1 + \omega) \Psi(\mathbf{z}_{t_{n+k}}, t_{n+k}, t_n, c) - \omega \Psi(\mathbf{z}_{t_{n+k}}, t_{n+k}, t_n, \emptyset)$
        \State $\tilde{\mathbf{z}}_{s_m} = f^{m}_{\theta} (\mathbf{z}^{\phi}_{t_{n+k}}, t_{n+k}, \omega, c)$ and $\hat{\mathbf{z}}_{s_m} = f_{\theta^-} (\mathbf{z}^{\phi}_{t_n}, t_n, \omega, c)$
        \State Obtain $\tilde{\mathbf{z}}_s$ and $\mathbf{z}_s$ through adding noise to $\tilde{\mathbf{z}}_{s_m}$ and $\mathbf{z}_0$
        \State $\mathcal{L}(\theta, \theta^-) = d(\tilde{\mathbf{z}}_{s_m}, \hat{\mathbf{z}}_{s_m}) + \lambda ( \text{ReLU}(1-f_D(\tilde{\mathbf{z}}_s,s,c))$
        \State $\theta \gets \theta - \eta \nabla_{\theta} \mathcal{L}(\theta, \theta^-)$
        \State $\theta^- \gets \text{stopgrad}(\mu \theta^- + (1 - \mu) \theta)$   
        \State $\mathcal{L}_{\theta_D} = \text{ReLU}(1 - f_D({\mathbf{z}}_s,s,c)) + \text{ReLU}(1 + f_D(\tilde{\mathbf{z}}_s,s,c))$
        \State $\theta_d \gets \theta_d - \eta \nabla_{\theta_d} \mathcal{L}_{\theta_d}$
        \Until convergence
    \end{algorithmic}
    \label{algo}
\end{algorithm*}

\subsection{Detailed Definitions of Evaluation Metrics}

In this section, we provide a more detailed explanation of the various metrics used in this paper to evaluate the performance of our models, and we discuss their respective roles. We divide the chosen evaluation metrics into five categories, as described below:

\textbf{Time Cost}: To gauge each model’s inference efficiency, we follow~\cite{chen2023executing,dai2025motionlcm} and report the \textbf{Average Inference Time per Sentence (AITS)}, measured in seconds. Specifically, we calculate AITS with a batch size of 1, excluding any time spent loading the model and dataset, in accordance with~\cite{dai2025motionlcm}.

\textbf{Reconstruction Quality:} Following~\cite{von2018recovering, chen2023executing}, we adopt the \textbf{Mean Per Joint Position Error (MPJPE)} to evaluate reconstruction performance. MPJPE represents the average Euclidean distance between the ground-truth and estimated joint positions, which can be calculated as:
\[
\text{MPJPE} = \frac{1}{N \times J} \sum_{i=1}^{N} \sum_{j=1}^{J} \| \hat{p}_{i,j} - p_{i,j} \|_2,
\]
Where $N$ denotes the number of samples, $J$ represents the number of joints, $\hat{p}_{i,j}$ is the predicted position of the $j$-th joint in the $i$-th sample, $p_{i,j}$ is the corresponding ground-truth position. A lower MPJPE value indicates a more accurate reconstruction, reflecting better alignment between the predicted and ground-truth joint positions. Additionally, following \cite{motionlcm-v2}, \textbf{Feature Error} is used to calculate the mean squared error between the encoder results generated from input motion sequence and reconstructed motion sequence, respectively. In formula: 
\[
\text{Feature error} = \frac{1}{N} \sum_{i=1}^{N} \| \mathcal{E}(\tilde{p}_i) - \mathcal{E}(\hat{p}_i) \|_2,
\]
where $\tilde{p}$ is input motion sequence, $\hat{p}$ is reconstructed motion sequence and $\mathcal{E}$ is encoder of the trained VAE.

\textbf{Condition Matching}: Motion/text feature extractors provided by~\cite{guo2022generating} encode motion and text into a shared feature space, where matched pairs lie close together. To compute \textbf{R-Precision}, each generated motion is mixed with 31 mismatched motions, and we measure the Top-1/2/3 text-to-motion matching accuracy in this shared space. Likewise, \textbf{MM Dist (Multimodal Distance)} calculates the average distance between each generated motion and its corresponding text prompt in the feature space, thereby reflecting how accurately the model captures the intended semantics.

\textbf{Motion Quality}: We use the \textbf{Frechet Inception Distance (FID)} to quantify how closely the distribution of generated motions aligns with real motions. Feature extraction is performed using the method described in~\cite{guo2020action2motion,guo2022generating}, and FID is then computed on these extracted features.

\textbf{Motion Diversity}: Following~\cite{guo2022tm2t,guo2020action2motion}, we use two metrics to assess the variety of generated motions: \textbf{Diversity} measures the global variation among all generated motions. Specifically, two subsets of the same size $S_d$ are randomly sampled from the generated set, with feature vectors $\{v_1, \ldots, v_{S_d}\}$ and $\{v'_1, \ldots, v'_{S_d}\}$. The average Euclidean distance between corresponding pairs $(v_i, v'_i)$ is then reported:
\[
\text{Diversity} = \frac{1}{S_d} \sum_{i=1}^{S_d} \|v_i - v'_i\|_2.
\]

\textbf{MultiModality (MModality)} assesses how much the generated motions diversify within each textual prompt. For a randomly sampled set of $C$ text prompts, two equally sized subsets of $I$ motions are drawn from the motions generated for the $c$-th prompt, resulting in feature vectors $\{v_{c,1}, \ldots, v_{c,I}\}$ and $\{v'_{c,1}, \ldots, v'_{c,I}\}$. The mean pairwise distance is:
\[
\text{MModality} = \frac{1}{C \times I}
\sum_{c=1}^{C} \sum_{i=1}^{I} \|v_{c,i} - v'_{c,i}\|_2.
\]
While Diversity reflects the overall variation among all generated motions, MModality focuses on the variation within each action type.

\subsection{More Qualitative Results}
 We present more qualitative results applying our method to text-conditioned motion synthesis in Figure~\ref{fig:more_qualitative}. Each example is accompanied by a video. 
\begin{figure*}
\centering
% \resizebox{0.8\textwidth}{!}{%
\renewcommand{\arraystretch}{1.2} % Adjust row height for better readability
\begin{tabular}{>{\centering\arraybackslash}m{0.22\textwidth}
                >{\centering\arraybackslash}m{0.22\textwidth}
                >{\centering\arraybackslash}m{0.22\textwidth}
                >{\centering\arraybackslash}m{0.22\textwidth}}
\hline

\small{\textit{A person \textcolor{red}{jumps} rope.}}& \small{\textit{A man \textcolor{red}{climbs up} steps.}} & \small{\textit{Someone \textcolor{red}{acting like} a bird.}} & \small{\textit{A person \textcolor{red}{bowing}.}}\\

\includegraphics[height=3.8cm,keepaspectratio]{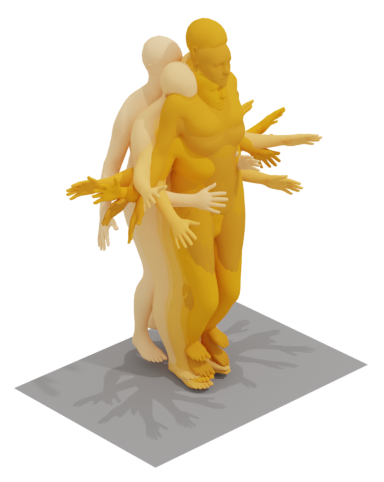} &
  \includegraphics[height=3.8cm,keepaspectratio]{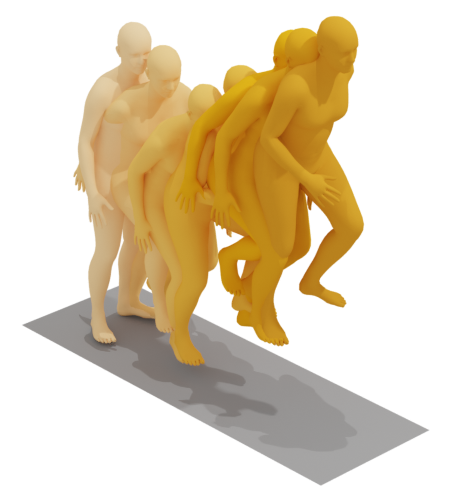} &
  \includegraphics[height=3.8cm,keepaspectratio]{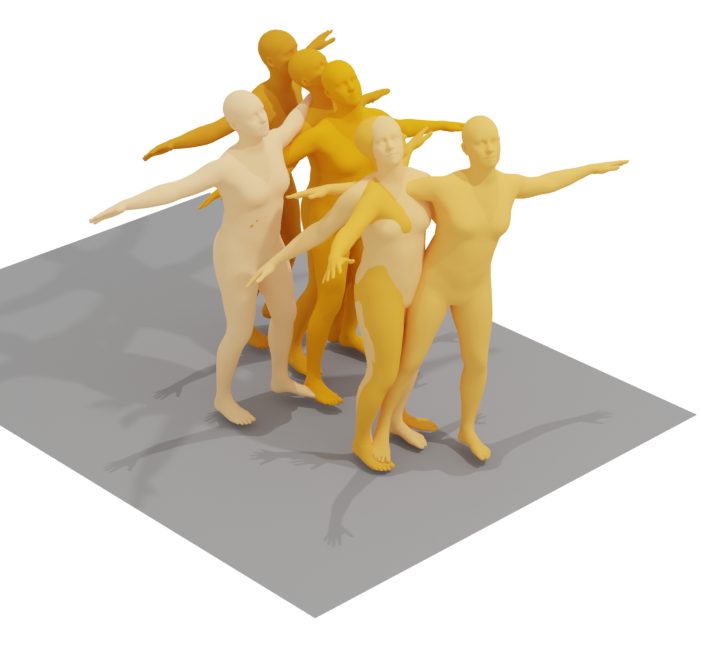} &
  \includegraphics[height=3.8cm,keepaspectratio]{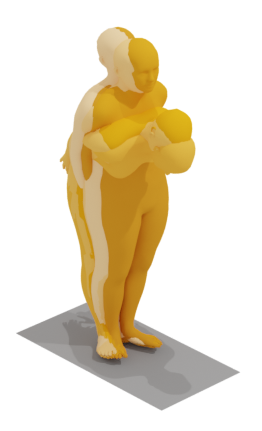} \\ 
\hline

\small{\textit{A person \textcolor{red}{walks} and \textcolor{red}{waves} their arms like a monkey.}}& \small{\textit{A person is \textcolor{red}{waving} his left hand.}} & \small{\textit{A person \textcolor{red}{walks} in a circle to his left.}} & \small{\textit{A person \textcolor{red}{sits down}.}}\\

\includegraphics[height=3.8cm,keepaspectratio]{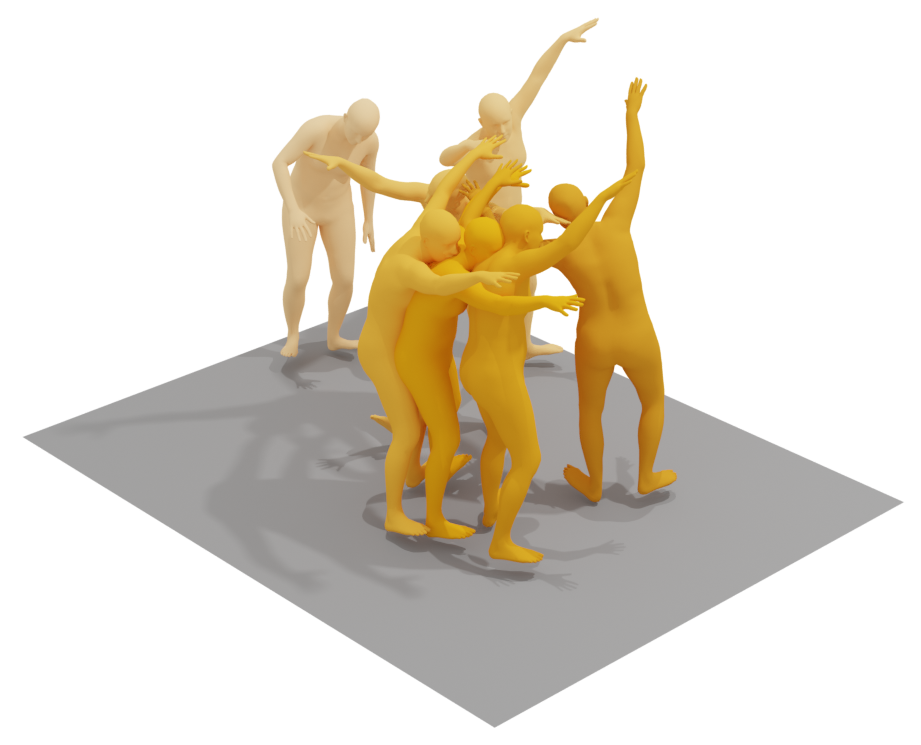} &
  \includegraphics[height=3.8cm,keepaspectratio]{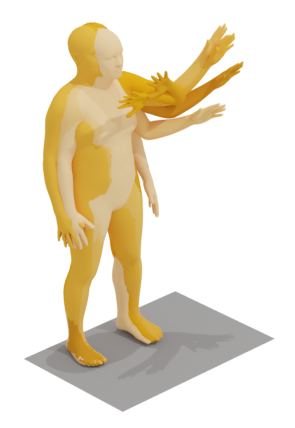} &
  \includegraphics[height=3.8cm,keepaspectratio]{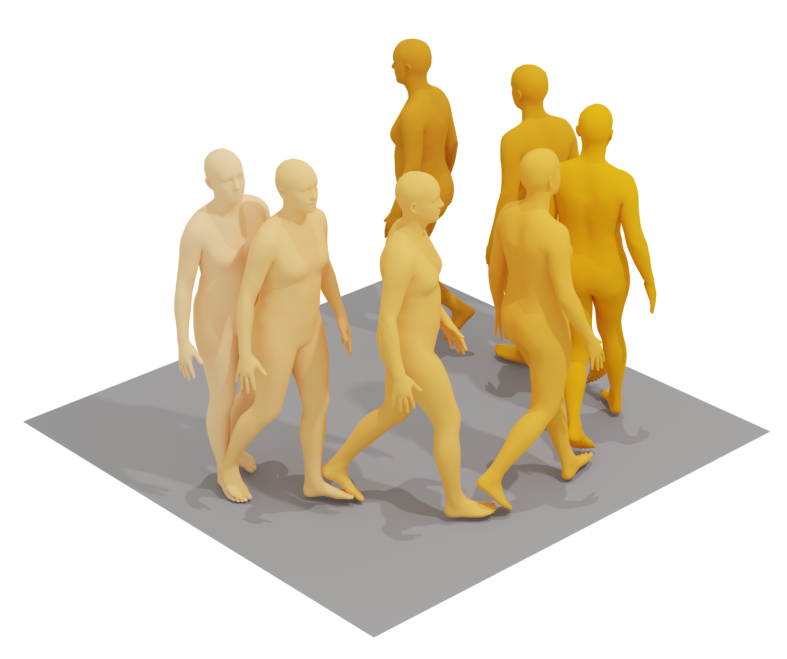} &
  \includegraphics[height=3.8cm,keepaspectratio]{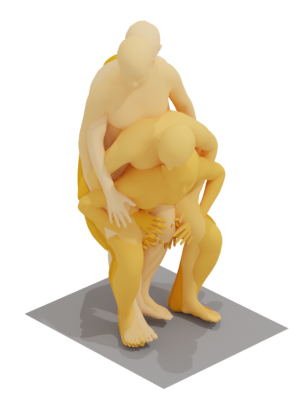} \\ 
\hline

\end{tabular}
%}
\bigskip
%\resizebox{0.8\textwidth}{!}{
\begin{tabular}{>{\centering\arraybackslash}m{0.32\textwidth}
                >{\centering\arraybackslash}m{0.32\textwidth}
                >{\centering\arraybackslash}m{0.32\textwidth}}

\small{\textit{A person is \textcolor{red}{running}.}}& \small{\textit{With arms out to the sides, a person \textcolor{red}{walks} forward.}} & \small{\textit{A person \textcolor{red}{walks} forward slowly, using one hand at a time to \textcolor{red}{pull} themselves forward or keep balance.}}\\

\includegraphics[height=3.8cm,keepaspectratio]{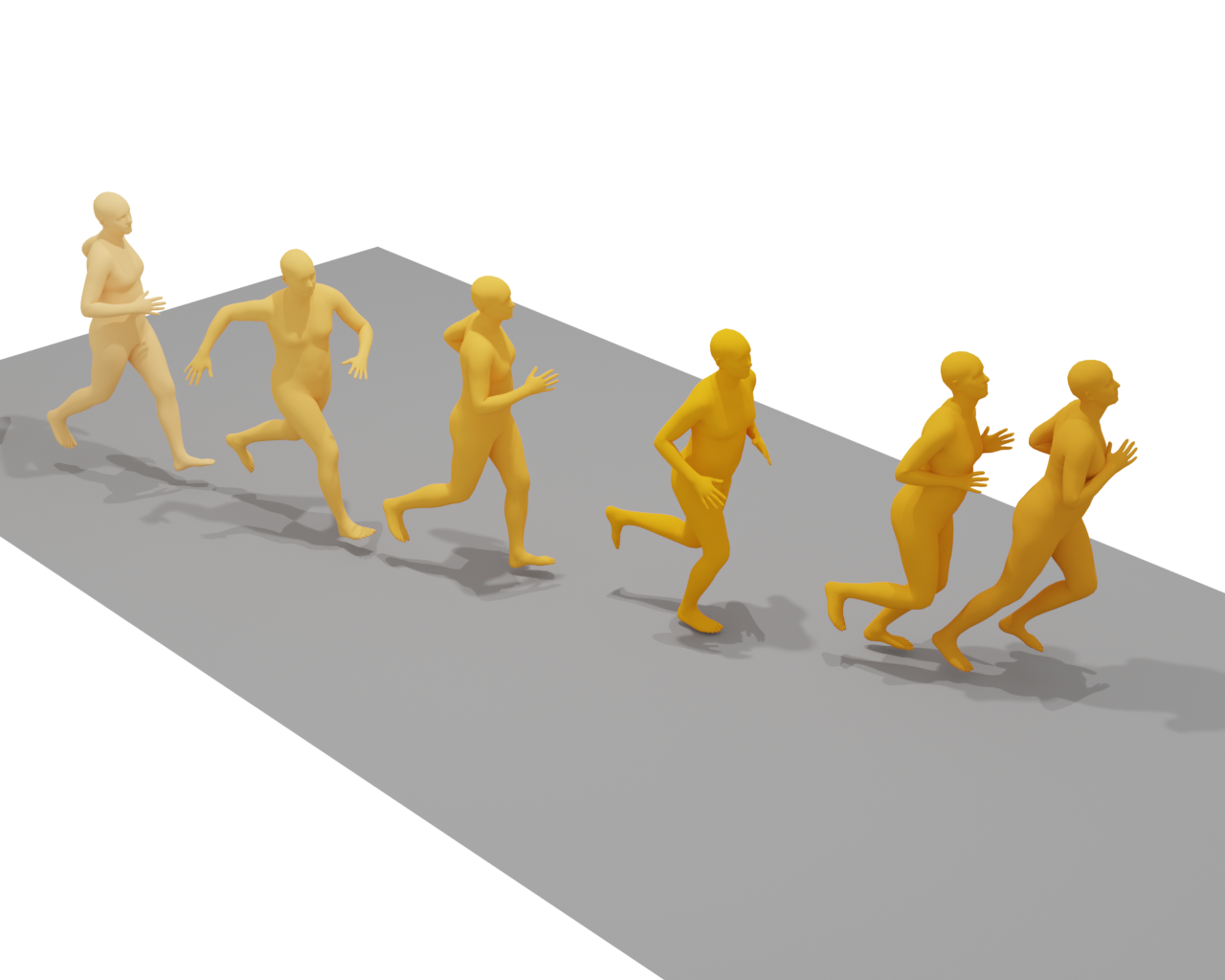} &
  \includegraphics[height=3.8cm,keepaspectratio]{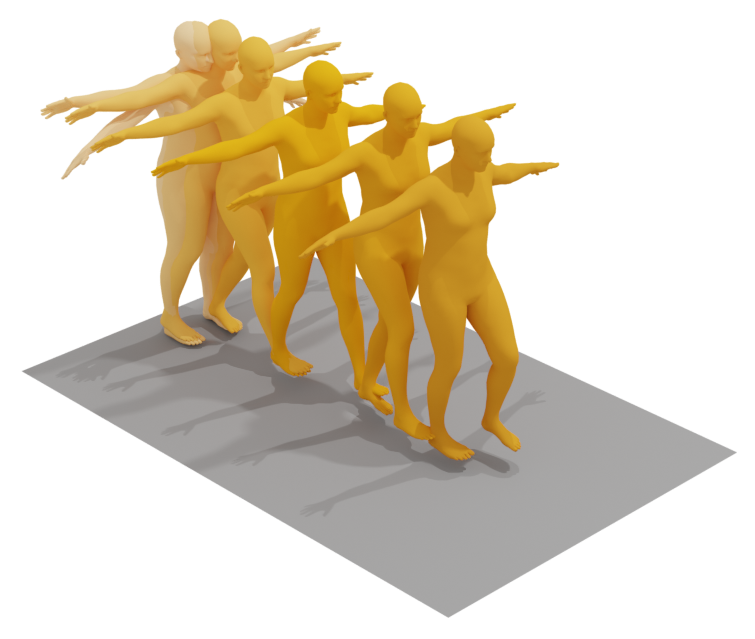} &
  \includegraphics[height=3.8cm,keepaspectratio]{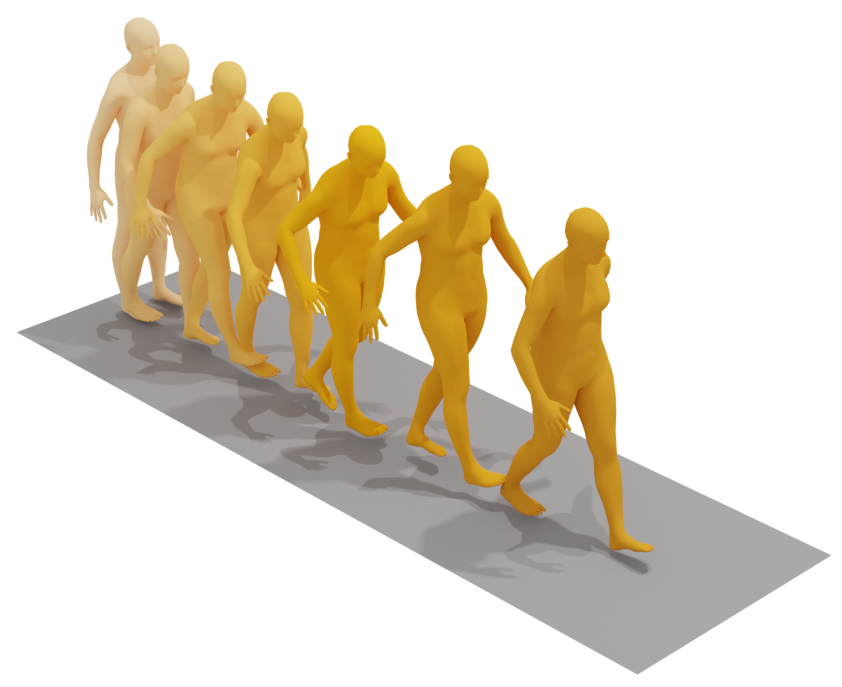} \\ 
\hline

\small{\textit{A man is \textcolor{red}{walking} in a large counter-clockwise circle.}}& \small{\textit{A person \textcolor{red}{takes a step} backward, carefully \textcolor{red}{gets down} on his hands and knees, then \textcolor{red}{crawls} forward.}} & \small{\textit{A person is \textcolor{red}{walking} backwards.}}\\

\includegraphics[height=3.8cm,keepaspectratio]{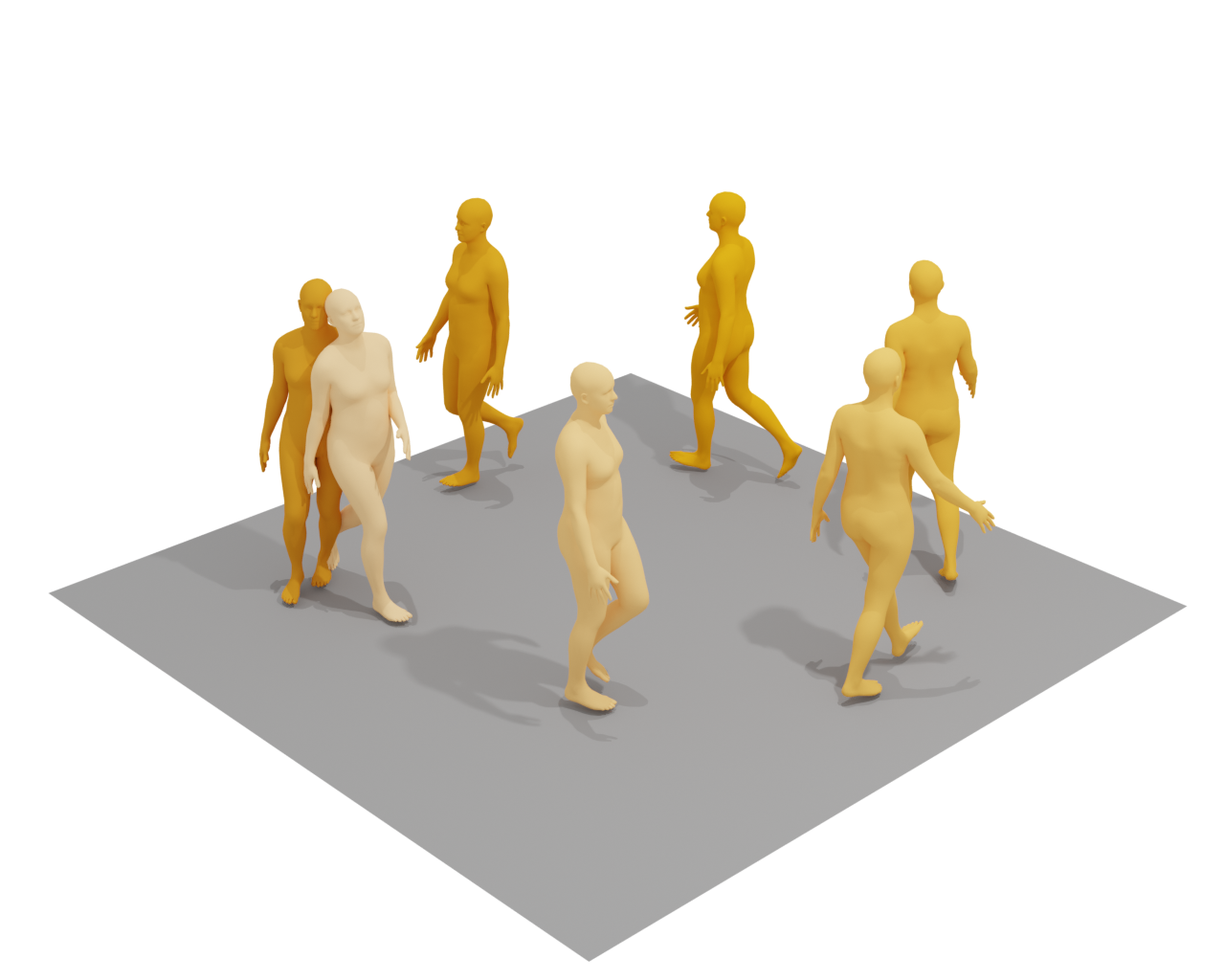} &
  \includegraphics[height=3.8cm,keepaspectratio]{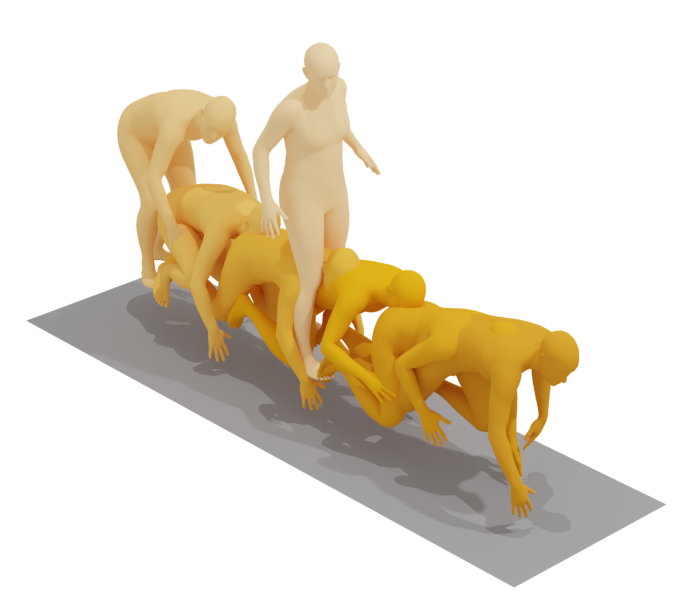} &
  \includegraphics[height=3.8cm,keepaspectratio]{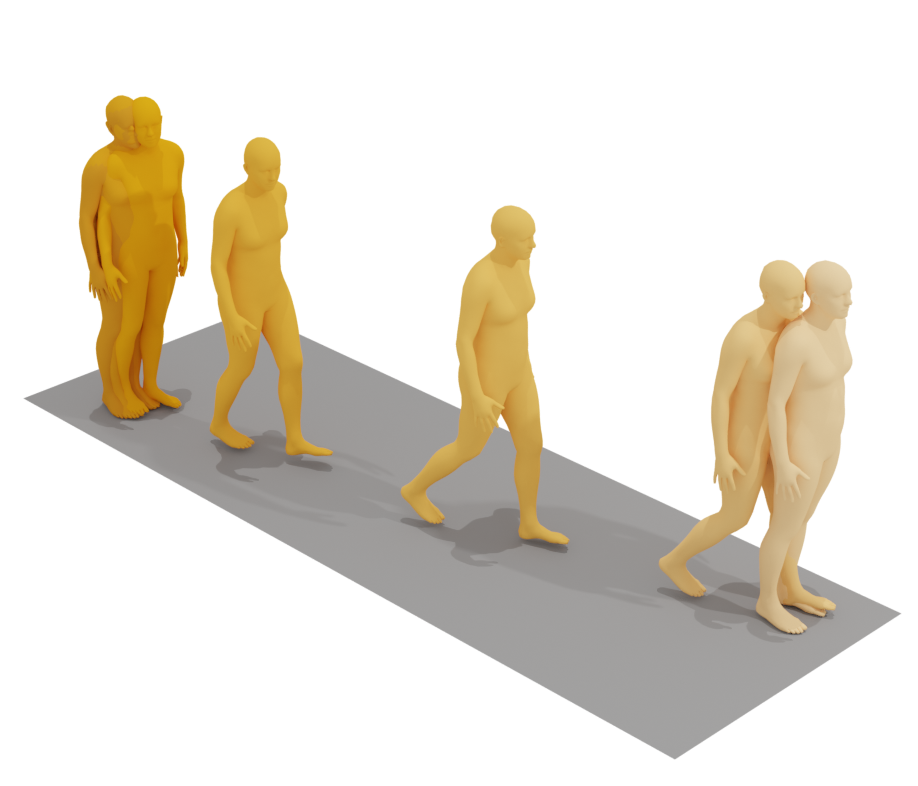} \\ 
\hline

\end{tabular}
%}
\caption{More samples generated from our proposed MotionPCM model. Lighter colours represent earlier time points.}
\label{fig:more_qualitative}
\end{figure*}

\end{document}